\documentclass[twoside]{article}
\usepackage{ecj,palatino,epsfig,latexsym,natbib}

\usepackage{amsmath}
\usepackage{amssymb}
\usepackage{dsfont}
\usepackage{subfig}
\usepackage{nicefrac}
\usepackage{url}
\usepackage[autostyle=true]{csquotes}
\usepackage{wasysym}

\usepackage{booktabs}
\usepackage{longtable}
\usepackage{pdflscape}
\usepackage{array}
\usepackage{multirow}
\usepackage{wrapfig}
\usepackage{float}
\usepackage{colortbl}
\usepackage{pdflscape}
\usepackage{tabu}
\usepackage{threeparttable}
\usepackage{threeparttablex}
\usepackage[normalem]{ulem}
\usepackage{makecell}
\usepackage{xcolor}
\usepackage{xspace}
\usepackage{dirtree}

\usepackage{algorithm}
\usepackage[noend]{algpseudocode}

\usepackage{tikz}
\usetikzlibrary{backgrounds, arrows, petri, topaths, calc, patterns}
\usepackage{tkz-berge}

\newtheorem{theorem}{Theorem}

\newtheorem{proof}{Proof}

\newcommand{\ie}{i.~e.\xspace}

\newcommand{\pkg}[1]{\texttt{#1}\xspace}

\begin{document}

\ecjHeader{x}{x}{xxx-xxx}{201X}{Sub-Graph-Based Mutation for Solving the Bi-Objective Minimum Spanning Tree Problem}{J.~Bossek and C.~Grimme}
\title{\bf On Single-Objective Sub-Graph-Based\\ Mutation for Solving the Bi-Objective Minimum Spanning Tree Problem}

\author{
        \name{\bf Jakob Bossek} \hfill \addr{bossek@aim.rwth-aachen.de}\\ 
        \addr{AI Methodology, Dept. of Computer Science, RWTH Aachen University, Germany}
\AND
       \name{\bf Christian Grimme} \hfill \addr{christian.grimme@wi.uni-muenster.de}\\ 
        \addr{Statistics and Optimization, Dept. of Information Systems, Univ. of M\"unster, Germany}
}

\maketitle

%\linenumbers
{
\parindent0pt
\textbf{Note} accepted for publication in the \emph{Evolutionary Computation Journal}~(ECJ)
}
\medskip

\begin{abstract}
We contribute to the efficient approximation of the Pareto-set for the classical $\mathcal{NP}$-hard multi-objective minimum spanning tree problem (moMST) adopting evolutionary computation. More precisely, by building upon preliminary work, we analyse the neighborhood structure of Pareto-optimal spanning trees and design several highly biased sub-graph-based mutation operators founded on the gained insights. In a nutshell, these operators replace (un)connected sub-trees of candidate solutions with locally optimal sub-trees. The latter (biased) step is realized by applying Kruskal's single-objective MST algorithm to a weighted sum scalarization of a sub-graph.

We prove runtime complexity results for the introduced operators and investigate the desirable Pareto-beneficial property. This property states that mutants cannot be dominated by their parent. Moreover, we perform an extensive experimental benchmark study to showcase the operator's practical suitability. Our results confirm that the sub-graph based operators beat baseline algorithms from the literature even with severely restricted computational budget in terms of function evaluations on four different classes of complete graphs with different shapes of the Pareto-front.
\end{abstract}


%%%%%%%%%%%%%%%%%%%%%%%%
%%%%% INTRODUCTION %%%%%
%%%%%%%%%%%%%%%%%%%%%%%%

\section{Introduction}
\label{sec:introduction}

Evolutionary algorithms (EAs) and bio-inspired algorithms in general are very successful in solving challenging real-world problems in many fields of applications, e.g., engineering, logistics, telecommunications, and design~\citep{D12,CWM12}. Furthermore, these type of randomized search heuristics are among the most successful when it comes to solving $\mathcal{NP}$-hard multi-objective optimization problems~\citep{DPAM02,beume2007,CCC2007}.

Given an undirected and weighted graph, the Minimum Spanning Tree (MST) problem is the challenge of finding a tree that maintains connectivity of all nodes and has minimal costs among all such trees. The MST problem is a fundamental graph-theoretic, well-studied combinatorial optimization problem, which can be solved efficiently, i.e., in polynomial time, by various algorithms~\citep{Prim57,Kr56}. The problem has countless applications, e.g., in the construction of communication spanning trees~\citep{hu1974}, medical imaging~\citep{An2000}, or many other areas that range from logistic via graph drawing to power grid network design~\citep{Atallah2009}. Additionally, it is used by sophisticated algorithms as a sub-routine to solve more challenging problems (e.g., the $\nicefrac{3}{2}$-approximation algorithm by \citet{Christofides76} for the metric Traveling-Salesperson-Problem (TSP) or for Minimum-Bottleneck optimization in graphs~\citep{Gabow1988BST}).
In many applications multiple weights are assigned to each edge leading to several conflicting objectives, which require simultaneous optimization turning the problem into a multi-objective one. Unfortunately, the multi-objective MST (moMST) was proven to be $\mathcal{NP}$-hard~\citep{CGM83}. Considerable effort was made in the classical algorithm and operations research community~\citep{Ehrgott2005,RH09}. However, classical, optimal algorithms, which guarantee to obtain the whole Pareto-set, e.g., the multi-objective extension of Prim's algorithm by \citet{Corley1985}, may get trapped by the intractability of the moMST~\citep{Hamacher1994}: the number of interim solutions may grow exponentially even if the Pareto-set is small, e.g., polynomially bounded in size. These approaches can, however, be used as blueprints for designing upper bounds in multi-objective branch and bound approaches for the problem, as, e.g., proposed in the context of a general branch and bound framework by~\citet{SourdS2008}, which is applied to the bi-objective MST problem. Sadly, the  work omits some implementation details in the algorithmic description for making it reproducible.
A recent approach by \citet{santos2018} transforms the original network into a directed network in which paths trees for the original network match paths in the transformed one. The authors then apply a labelling algorithm for finding multi-objective shortest path to solve moMST instances for arbitrary objectives. Still, the investigated problem instances are rather small and source code for reproducing results is not provided.

\citet{ZG99} were the first to introduce a genetic algorithm for the moMST. Their algorithm relies on an early version of the Non-dominated Sorting Genetic Algorithm (NSGA,~\citet{SD94}) and the elegant \emph{Pr\"ufer-number} for representation: a spanning tree is encoded via an integer vector from $\{1, \ldots, n\}^{n-2}$ where $n$ is the number of nodes.\footnote{\citet{Pruefer1918} gave an alternative prove of Cayley's theorem~\citep{Cayley1889} -- a complete graph has $n^{n-2}$ distinct spanning trees -- by showing that there is a bijection between the set of spanning trees and the set of integer vectors $\{1, \ldots, n\}^{n-2}$.} This representation allowed the authors to use standard evolutionary operators. Despite its obvious appeal, the Pr\"ufer-encoding was shown to be a poor representation for spanning trees since small changes to the encoding may result in large changes of the encoded spanning tree. This was first shown empirically by \citet{KC01} and later theoretically verified by \citet{Gottlieb2001}. Both papers suggest that a direct encoding, i.e. representing a tree by its set of $(n-1)$ edges, is more beneficial. Note that a recent study by \cite{barbosa2020DataStructuresForDirectSpanningTree} investigates multiple data structures for the direct representation of very large graphs together with respective variation operators. They find their proposed data structure very efficient for graphs with $> 10\,000$ nodes but also confirm that the direct encoding in arrays is performing well for small and medium instance.

If domain or problem-specific knowledge is available it is well-known that incorporation of this knowledge into the evolutionary process can have strong benefits. This can be accomplished by heavily problem-tailored variation operators (see, e.g.,~\citet{LTHGH08}) or the integration of biased local-search procedures (see, e.g.,~\citet{Ishibuchi2009}, who developed a memetic algorithm that puts higher probability on promising solutions in the local neighborhood of the local search operator).

In terms of graph problems, \citet{RKJ2006_BiasedMutationOperators} investigated the relationship between the rank of an edge and its occurrence in optimal solutions for the single-objective versions of the MST, the degree-constraint MST, and the TSP. In their work, they derived several tailored mutation operators, which put higher selection probability on low-rank edges and show both experimentally and by means of rigorous runtime analysis that these algorithms outperform the corresponding non-tailored operators by orders of magnitudes. Recently, \citet{BNG2019} transferred Raidl's work to the multi-objective domain. The authors experimentally verified that biased mutation can speed up the evolutionary search process for the moMST. They used different multi-objective ranking schemes to alter the probability distribution of edge selection for the 1-edge-exchange mutation operator, e.g., by putting higher selection probabilities on non-dominated edges. 

In the context of bi-objective MST problems, \citet{FernandesMGG2019} propose the application of a transgenetic algorithm, where so-called transgenetic agents are responsible for exchanging solution information and recombining them to new bi-objective MST solutions using weighting of objectives and single-objective MST approaches to construct new MSTs. Previous work by the authors of this paper~\citep{BG2017ParetoBeneficial} specifically focuses on mutation operators in standard meta-heuristics and introduces a heavily problem-tailored mutation operator -- called the sub-graph mutation (SG) -- for the bi-objective MST problem. The idea originates from an analysis of Pareto-optimal solutions and can be wrapped up in a nutshell as follows: given a spanning tree, (1) select a connected sub-tree at random and (2) replace it with an optimal sub-tree with respect to a randomly selected objective. The latter is achieved by applying a single-objective MST algorithm on the sub-graph induced by the sampled sub-tree. Hence, their operator can be seen as a hybrid of mutation and local-search. Experiments reveal fast convergence towards the Pareto-front for a class of complete graphs with two real-valued objectives sampled from a uniform distribution.

Here, we will build on our conference paper~\citep{BG2017ParetoBeneficial} and extend it considerably by designing and experimentally evaluating two generalizations of the SG operator. Before we detail the contribution and structure of this work, we introduce the necessary notations and definitions in the context of the multi-objective minimum spanning tree problem next.

%%%%%%%%%%%%%%%%%
%%%%% MOMST %%%%%
%%%%%%%%%%%%%%%%%

\section{The Multi-Objective Spanning-Tree-Problem}
\label{sec:problem}
We consider undirected weighted graphs $G = (V, E, c)$ with $n = |V|$ nodes and $m = |E|$ edges. A vector-valued cost function $c : E \to \mathbb{R}^q_{>0}, \ c(e) = (c_1(e), \ldots, c_q(e))^{\top}$ assigns a positive weight to each edge $e \in E$. Given a subset $S \subset V$ we call $G[S]$ the \emph{sub-graph induced in $G$ by $S$}, which is the graph comprising the node set $S$ and the subset of edges from $E$ for which both end nodes are in $S$.
Each acyclic, connected sub-graph $T = (V, E_T)$ with $E_T \subset E$ is termed a \emph{spanning tree} of $G$. In the following, we occasionally identify a tree by its edge set $E_T$. Moreover, we use the notation $V(G)$ and $E(G)$ to refer to the node set and edge set of a graph $G$.
Further, $\mathcal{T}$ denotes the set of all spanning trees of $G$. With slight abuse of notation\footnote{The domain of $c$ is actually the Cartesian product $V^2$ and not a set of trees.} the cost-vector of a tree $T \in \mathcal{T}$ contains the component-wise sum of edge weights of $T$, \ie, $c(T) = (c_1(T), \ldots, c_q(T))^{\top}$ with $c_i(T) := \sum_{e \in E_T} c_i(e)$. Then the problem
\begin{align*}
  \min_{T \in \mathcal{T}} c(T) = (c_1(T), \ldots, c_q(T))^{\top}
\end{align*}
is termed the \emph{multi-objective minimum spanning tree problem} (moMST). For $q = 1$, i.e., in the single-objective case, the solution is obvious. However, since there is no canonical order in $\mathbb{R}^q, q \geq 2$, we adopt the concept of Pareto-dominance~\citep{CCC2007,miettinen1999nonlinear} to define optimality in the multi-objective case. Here, the goal is to find a set of incomparable trade-off solutions, so-called \emph{efficient} or \emph{Pareto-optimal} solutions, $PS = \{T \in \mathcal{T} \, | \, \nexists\ T' \in \mathcal{T}: c(T') \preceq c(T)\}$, termed \emph{Pareto-set} and its image $c(PS) = \{c(T) \, | \, T \in PF\}$, termed the \emph{Pareto-front}. Here $\preceq$ is the \emph{dominance relation}. A solution $T$ dominates another solution $T'$, $T \preceq T'$, if $c_i(T) \leq c_i(T') \,\forall i = 1,\ldots,q$ and $\exists j \in \{1, \ldots, q\}$ with $c_j(T) < c_j(T')$.
A simple method to obtain efficient solutions to multi-objective problems is via scalarization where the original multi-objective problem is transformed into a single-objective problem. A well known scalarization method is the \emph{weighted-sum} method. Let $\lambda_1, \ldots, \lambda_q > 0$ with $\sum_{i=1}^{q} \lambda_i = 1$. Then the weighted-sum approach minimizes
\begin{align}
    \min_{T \in \mathcal{T}} \sum_{i=1}^{q} \lambda_i c_i(T),\label{eq:weighted-sum}
\end{align}
where trees which can be obtained by optimizing Eq.~\eqref{eq:weighted-sum} are called \emph{supported (efficient)}; these trees are provably Pareto-optimal~\cite[p.~71]{Ehrgott2005}.
All other Pareto-optimal trees are called \emph{unsupported (efficient)}~\citep{RH09}.

%%%%%%%%%%%%%%%%%%%%%%%%
%%%%% CONTRIBUTION %%%%%
%%%%%%%%%%%%%%%%%%%%%%%%

\section{Contribution and Structure of the Paper}
\label{sec:structure}
In this paper we continue and significantly extend the work initiated in \cite{BG2017ParetoBeneficial}. Our contributions are as follows: 
\begin{enumerate}
    \item We extend the analysis of optimal solutions to different graph classes in order to gain comprehensive insights into the properties of Pareto-optimal solutions and to evaluate the developed approaches on a wide range of problem instances. Moreover and in contrast to our earlier work~\citep{BG2017ParetoBeneficial}, we analyse supported efficient solutions of the Pareto-front for instances with $n > 10$ nodes to find stronger support for subsequent operator design. 
    \item We propose two generalizations of the original sub-graph mutation (SG) operator allowing to (a) optimize a scalarized sub-problem using a weighted-sum approach and (b) to replace a set of not necessarily connected edges during mutation. Both adaptations allow for a more flexible search process.
    \item We provide bounds on the time complexity of the introduced operators.
    This is helpful for the subsequent evaluation of these more sophisticated mutation approaches and their performance compared to other operators.
    \item We verify their effectiveness in an experimental study considering four graph classes with different characteristics of the Pareto-front (globally convex, concave, many/few Pareto-optimal solutions) for benchmarking.
\end{enumerate}
 
 Our results point out a major weakness of the original SG operator proposal and the benefits of the introduced modifications in comparison with (biased) 1-edge-exchange, Zhou \& Gen's genetic algorithm~(GA; \citet{ZG99}) as well as a simple weighted-sum approach. We show that the sub-graph based operators converge much faster towards the Pareto-front with respect to function evaluations and in terms of several multi-objective performance indicators (hypervolume indicator, unary $\varepsilon$-indicator and $\Delta p$). However, this comes at the cost of higher computational complexity for mutation itself. Finally, we sketch ideas on how to mitigate the computational overhead. 
 
 Note that this work specifically focuses the integration of (helpful) problem knowledge via mutation operators on edge-list representation. Thus, we are primarily interested in their individual contribution in generating (near) optimal solutions inside an encapsulating meta heuristic. Consequently, we test a plethora of mutation operators inside only a single meta-heuristic, which is one of the most established, simple (and well understood) meta-heuristics in the EMO domain: NSGA-II~\citep{DPAM02}.
 A comprehensive survey of the operators' behaviour in all modern or available meta-heuristics is beyond the scope and goal of this paper. Also note that whenever we speak of \emph{algorithms} in this paper, we refer to combinations of a specific encapsulating meta-heuristic (namely NSGA-II) and a specific variation operator applied inside.
 
 The remainder of this paper is structured as follows: after a detailed description of our instance generation approach for benchmarking in Section~\ref{sec:instances}, we first present a detailed analysis of Pareto-optimal solutions in Section~\ref{sec:analysis}. In a second step and based on the gained insights from Section~\ref{sec:analysis}, we deduce sub-graph based mutation operators in Section~\ref{sec:mutation}, discuss their runtime properties, and show that they are Pareto-beneficial. In Section~\ref{sec:experiments}, we describe our experimental setup and detail the results. Finally, in Section~\ref{sec:conclusion}, the findings are summarized and future directions for research are highlighted.

%%%%%%%%%%%%%%%%%%%%%
%%%%% INSTANCES %%%%%
%%%%%%%%%%%%%%%%%%%%%

\section{Instance Generation}
\label{sec:instances}

We considerably widen the set of benchmark instances in contrast to our preliminary work~\citep{BG2017ParetoBeneficial}: We consider different groups of graphs with different shapes of the Pareto-front, which pose difficulties for multi-objective optimization algorithms. 
Note that -- in par with existing work~\citep{KC2001Generators, ZG99} -- we limit our experimental study to the bi-objective case here. We plan to expand our work to $q \ge 3$ in future work though. For instance generation we used the multi-step graph generator \texttt{grapherator}~\citep{B2018grapherator}.\footnote{GitHub repository: \url{https://github.com/jakobbossek/grapherator}} All benchmark graphs are complete, hence exhibiting a search space of maximum cardinality $n^{n-2}$ according to Cayley's theorem~\citep{Cayley1889}. We are aware of different (more sparse) interconnection structures being also of interest, but we find the shape of the Pareto-front to be the most interesting aspect for this study.
For later ease of reference we introduce four classes -- abbreviated C1 to C4 -- in the following. The reader is advised to consider Fig.~\ref{fig:exemplary_instances} as descriptive examples from each class while reading the following descriptions.
The figure depicts the edge costs for exemplary instances of each class alongside an approximation of the Pareto-front; the latter is obtained by weighted-sum scalarisation that is described in more detail in Section~\ref{sec:analysis}.
\begin{figure*}[thb]
  \centering
  \includegraphics[width=\textwidth]{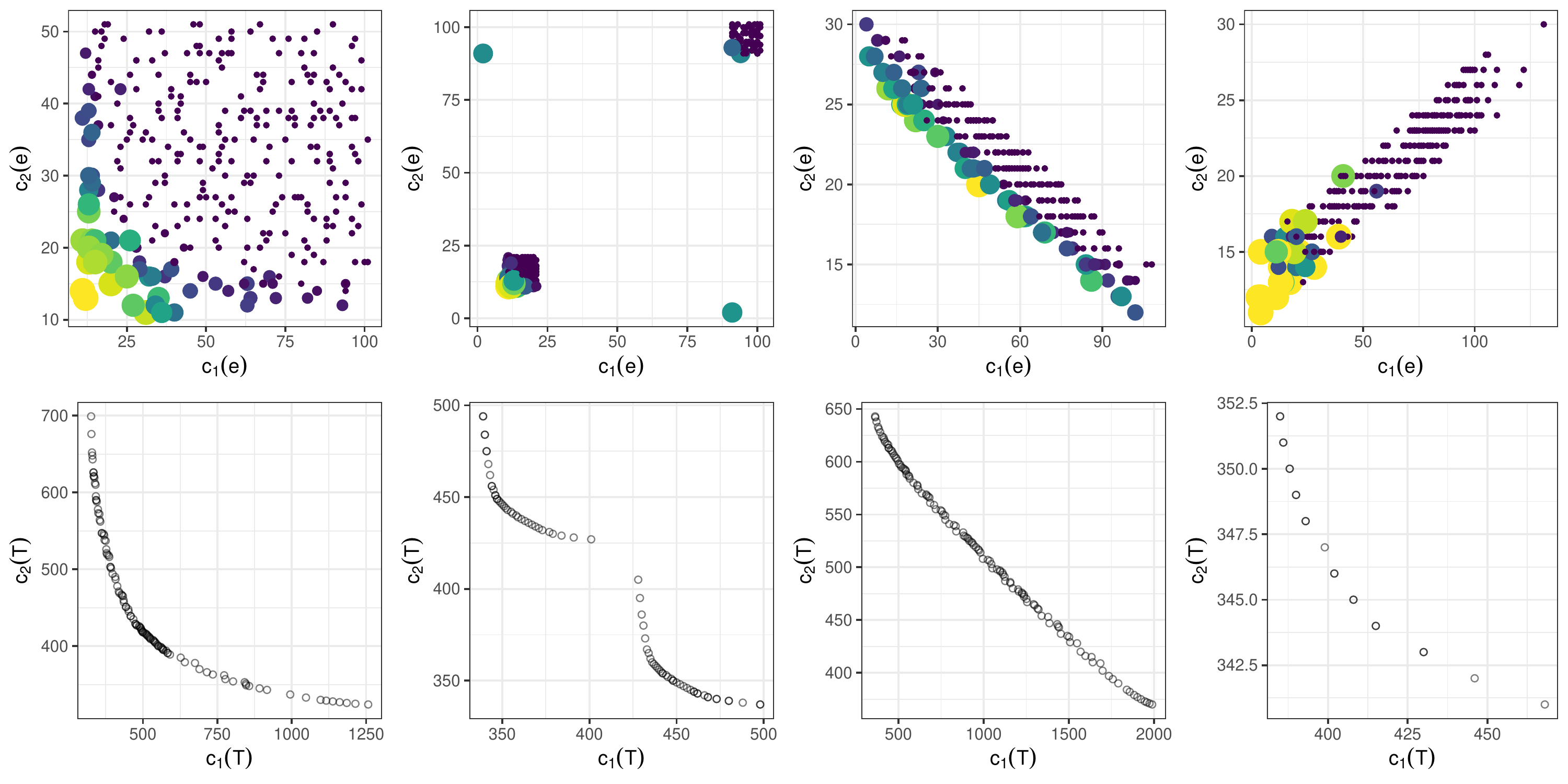}
  \caption{Visual examples for each one instance ($n = 25$) of each benchmark graph class C1 to C4 (from left to right). Depicted are scatterplots of the edge weights coloured and sized by their respective occurrence in supported efficient solutions (top row; the lighter/larger, the more frequent) and scatterplots of a Pareto-front approximation (bottom row).}
  \label{fig:exemplary_instances}
\end{figure*}
\begin{description}
\item[C1] \enquote{Classical} instances as used in~\citet{ZG99} and \citet{KC01}. Here, the first weight $c_1$ is sampled from a $\mathcal{U}(10, 100)$ distribution, while $c_2$ stems from a $\mathcal{U}(10, 50)$ distribution. The Pareto-front has a globally convex structure with many solutions nested in small concave regions. 
\item[C2] Here, the first weight is the rounded Euclidean distance between the points in the Euclidean plane. The second weight is sampled in a way that leads to a Pareto-front with a large concave region. The procedure was originally introduced by \citet{KC2001Generators} for the degree-constraint MST. Note that a concave shape poses major difficulties to approaches that are capable of locating \emph{supported efficient solutions} only, e.g., the classic weighted sum scalarization approach~\citep[Chapter~3, pp.~65]{Ehrgott2005}.
\item[C3] In this class the first weight is again given by the (rounded) Euclidean distance of the node coordinates. The second weight is chosen such that there is a strongly negative correlation ($\rho \approx -0.95$) between the first and the second weight. Instances with this weight structure show a very large number of efficient solutions and are in particular challenging for exact methods.
\item[C4] This class is the counterpart to C3 with a highly positive weight correlation ($\rho \approx 0.95$). Instances expose a rather limited number of efficient solutions.
\end{description}
For each graph class and number of nodes $n \in \{25, 50, 100, 250\}$ we generated $10$ instances resulting in a total benchmark set size of $160$ instances.

%%%%%%%%%%%%%%%%%%%%%%%%%%%%%%%%%%%%%%%%%
%%%%% ANALYSIS OF OPTIMAL SOLUTIONS %%%%%
%%%%%%%%%%%%%%%%%%%%%%%%%%%%%%%%%%%%%%%%%

\section{Analysis of Supported Efficient Solutions}
\label{sec:analysis}

Although mutation operators can be considered as important drivers for innovation in evolutionary algorithms~\citep{Bey01}, application often ignores available problem knowledge and relies on pure and unbiased randomness.\footnote{This approach is certainly recommended for black-box problems, however, for known problem characteristics, this approach wastes valuable information.} Here, we analyze our problem class of multi-objective spanning trees and try to integrate domain knowledge from observed and supported efficient solutions into the construction of mutation operators. 

As mutation works on the level of problem representations, we analyse the genotypic characteristics of \emph{supported efficient solutions} for the moMST and follow the approach from our former work~\citep{BG2017ParetoBeneficial}. For each set of benchmark instances (see Section~\ref{sec:instances}), we compute the supported efficient solutions using a simple weighted-sum approach. Therefore, we generate $5\,000$ equidistantly spread weights $\lambda_k \in [0, 1], k = 1, \ldots, 5\,000$ and calculate the optimal solution for $G = (V, E, \lambda_k c_1 + (1 - \lambda_k)c_2)$ and each $k$. In principle, the complete Pareto-set/front can be computed by adopting the exact algorithm proposed by \citet{Corley1985}, a multi-objective generalization of Prim's well-known single-objective MST algorithm (see also \citet[p. 208ff]{Ehrgott2005}). However, the algorithm gets trapped by an exponential number of interim solutions and hits the wall clock time and/or memory limitations prior to completion on almost all instances (except for a handful of C4 instances) and does not find solutions which could not be discovered by the weighted sum algorithm. Note that the supported efficient solutions do not comprise solutions from the concave parts of the Pareto-front. This limits the generality of our following analysis to the classes C1, C3, and C4. For C2, only solutions of the convex parts of the Pareto front could be analyzed.

For an empirical analysis of the supported efficient solutions, we focus on the neighborhood of solutions as well as on the similarities of solutions in terms of prevailing and redundant edges. While the first aspect of neighborhood allows conclusions on the possible existence of an evolutionary path via neighboring solutions, the latter can provide information on the benefit of purely random mutation, which activates or deactivates any edge by chance.

In order to measure the neighborhood relation of solutions $T_1$ and $T_2$, we employ the \emph{\underline{n}ormalized \underline{n}umber of \underline{c}ommon \underline{e}dges} (NNCE) metric, which is defined as
\begin{align*}
  \text{NNCE}(T_1, T_2) := \frac{|E_{T_1} \cap E_{T_2}|}{|E_{T_1}|} \in [0, 1]
\end{align*}
where $E_{T_1}$ and $E_{T_2}$ are the respective edge sets.

\begin{figure}[tb]
  \centering
  \includegraphics[width=0.55\textwidth, clip]{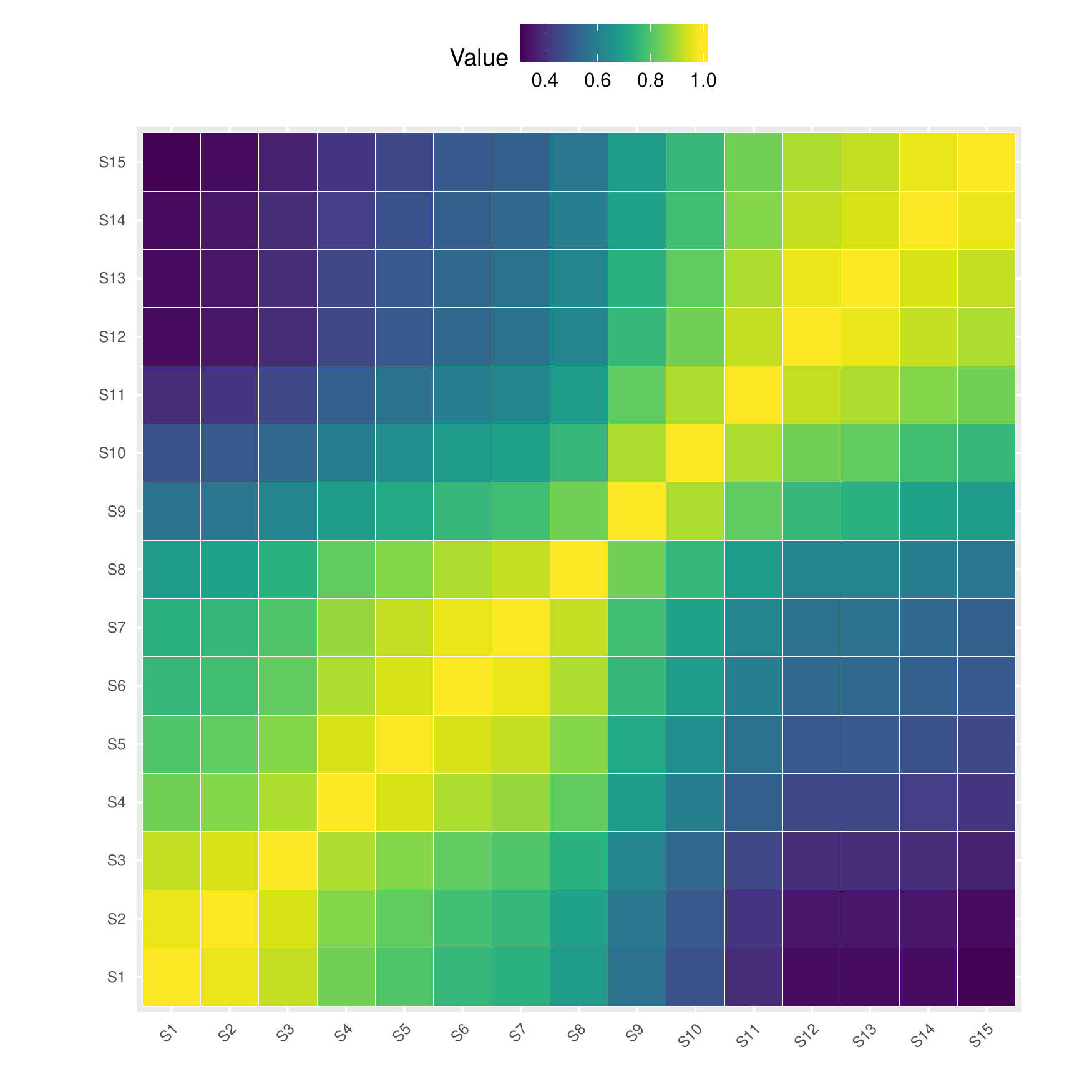}
  \caption{Exemplary heatmaps of pairwise similarity with respect to the NNCE between supported Pareto-optimal solutions for a graph with $n = 50$ nodes of class C2. Solutions are displayed in ascending order regarding objective one and thus in descending order regarding objective two.}
  \label{fig:analysis_heatmaps}
\end{figure}

For a complete comparison of solutions in a solution set, we can depict the pairwise comparison of solutions regarding NNCE as shown in Fig.~\ref{fig:analysis_heatmaps}. Each colored cell represents the NNCE value for the comparison of two solutions in the supported efficient solution set. Along the axes, solutions are shown in ascending order according to objective one (and hence in descending order of objective two), i.e., neighboring solutions in Fig.~\ref{fig:analysis_heatmaps} are also neighbors on the Pareto-front.

This analysis shows (also via all other instances and classes) that solutions have a strong neighboring relationship. The larger the distance of solutions in objective space is, the less similar the solutions are with respect to common edges. Interestingly, even rather contradicting solutions (those near the lexicographic extrema of the objectives) seem to have edges in common.

This finding is confirmed---again for all instances and classes of considered problems---by a second analysis based on a detailed investigation of solution phenotypes (i.e., spanning trees). For a given problem instance, see exemplarily Fig.~\ref{fig:analysis_embeddings} for instances of three classes, the occurrence of each edge in all supported efficient solutions is depicted. The frequency of occurrence is represented by the line style (the thicker the line, the more frequent the edge in the whole set of solutions) as well by the relative frequency annotated with each edge. Not displayed edges are not part of any solution. The constant use or omission of edges is strongly connected to the weighting of edges. As observed already in Fig.~\ref{fig:exemplary_instances} (top row), the dominance relationship of edges often determines their appearance in efficient solutions: the more a cost vector dominates others, the more often it can appear in efficient solutions.   

\begin{figure}[t]
  \centering
  \includegraphics[width=0.37\textwidth, trim = 0mm 15mm 0mm 15mm, clip]{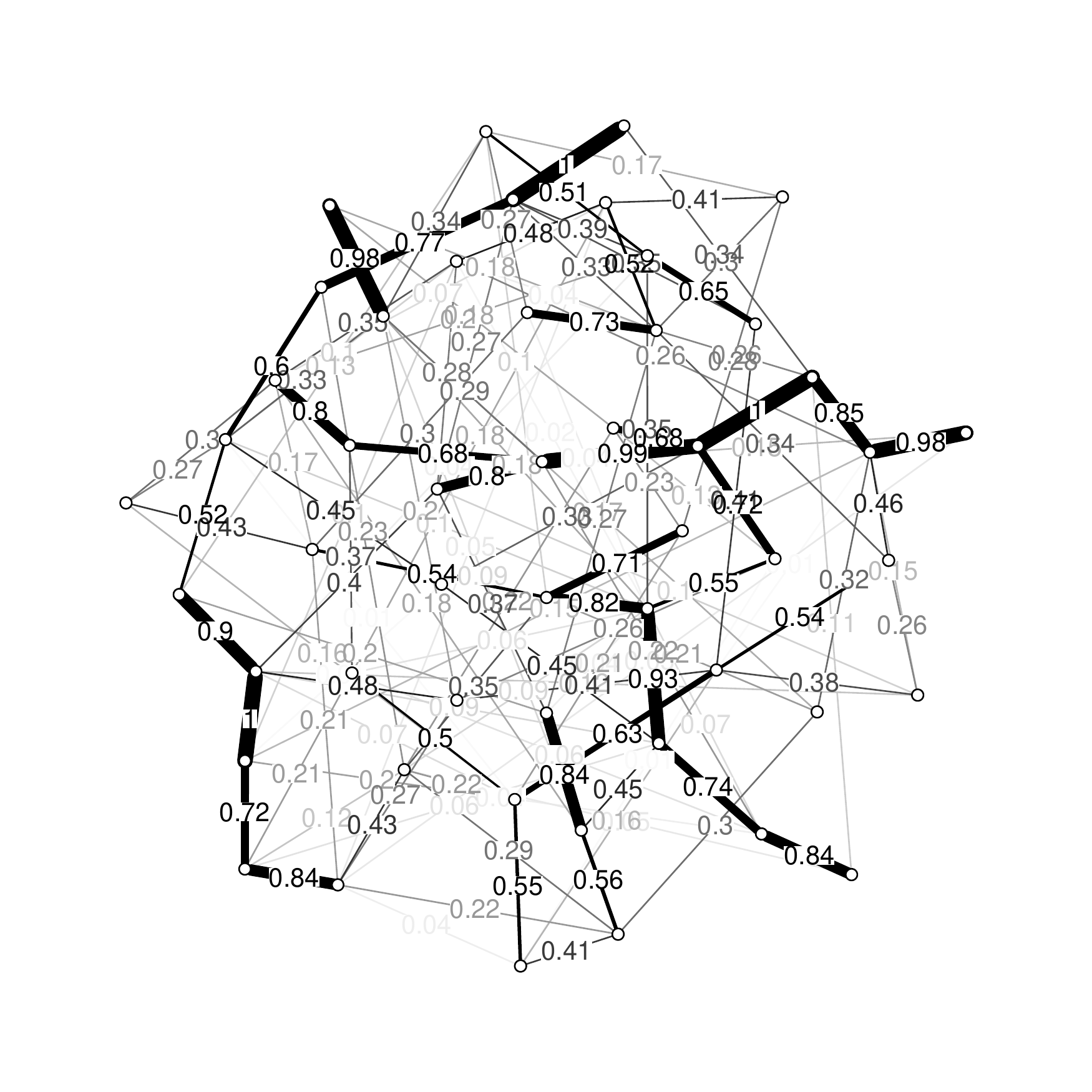}
  \hspace{-1cm}
  \includegraphics[trim = 0mm 15mm 0mm 15mm, clip, width=0.37\textwidth]{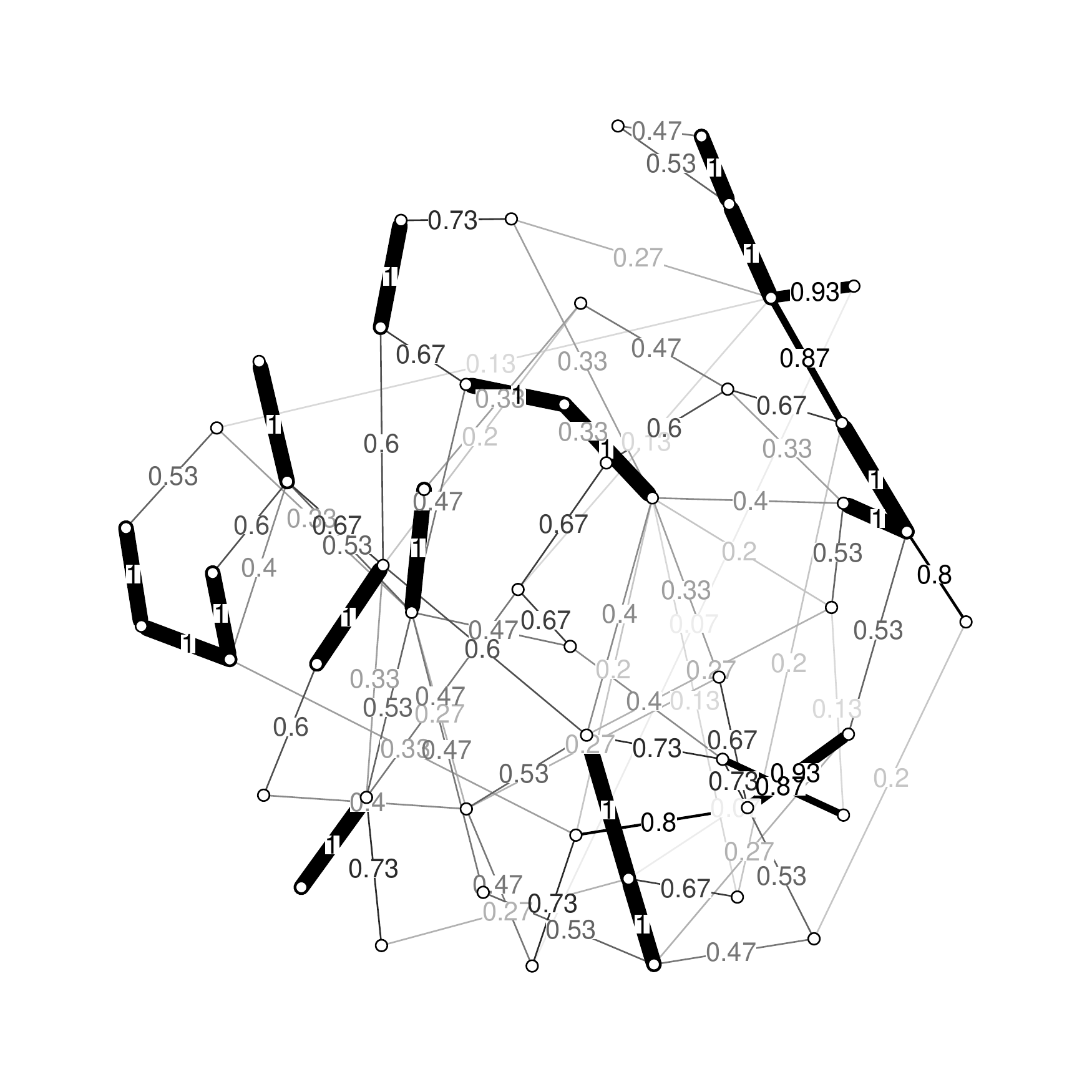}
  \hspace{-1cm}
  \includegraphics[trim = 0mm 15mm 0mm 15mm, clip, width=0.37\textwidth]{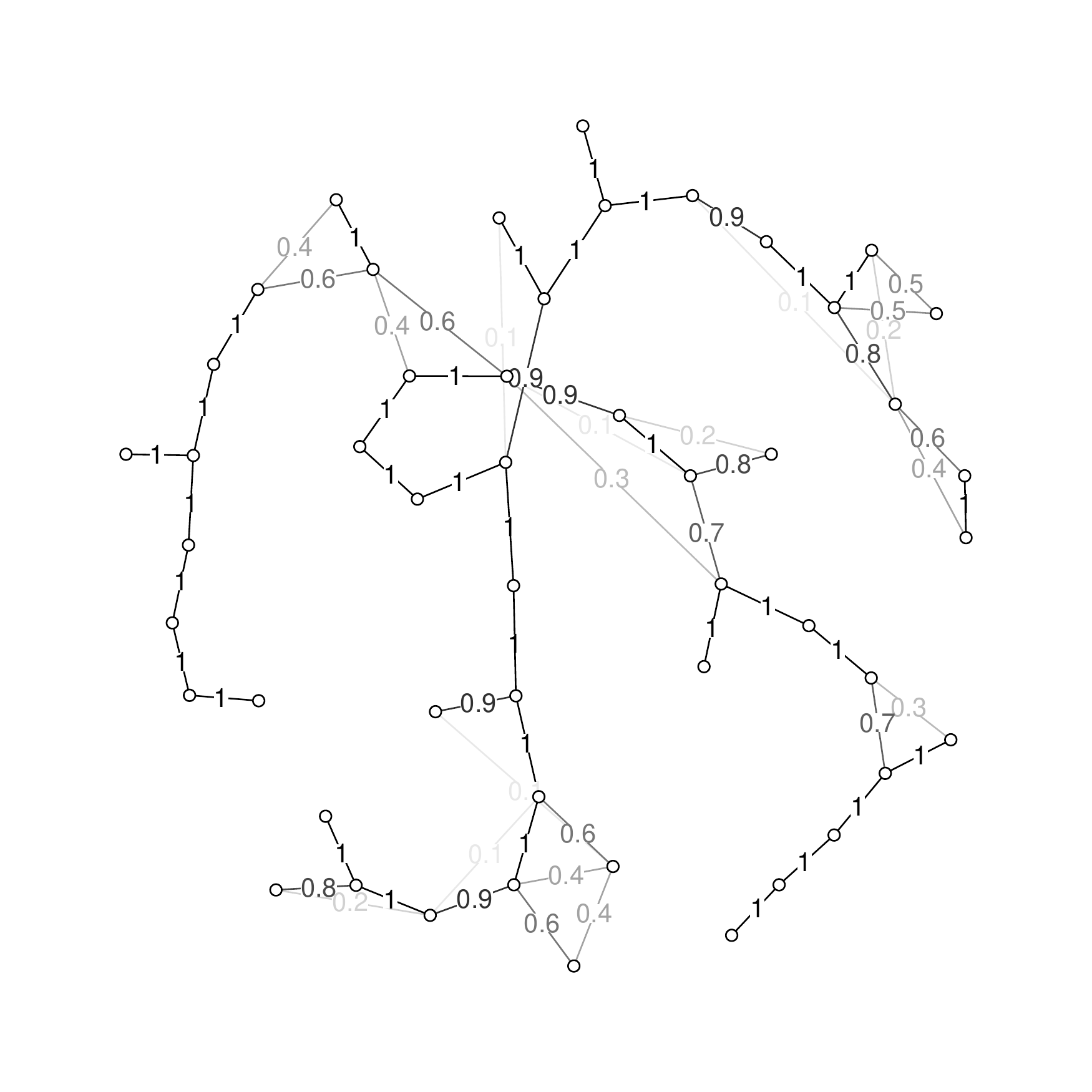}
  \caption{Representative embeddings of each one instance of class C1, C2 and C4 with $50$ nodes (from left to right). Edges are annotated with the fraction of supported Pareto-optimal solutions they are part of (also indicated by the edge thickness). Edges which do not occur in any solution are not displayed at all. We omit an example for C3, since due to the large number of Pareto-optimal solutions the plot is visually uninformative, because almost all edges are part of at least one supported efficient solution.}
  \label{fig:analysis_embeddings}
\end{figure}

This analysis shows that the use of simple edge exchange mutation may be disadvantageous. If an edge is part of or excluded from all (supported) efficient trees, the omission or inclusion is certainly disadvantageous, respectively. However, the simple random character of edge exchange mutation will lead to the removal or inclusion over the long run and deteriorate the solution. This destroys possibly good solutions and slows down the process of convergence to optimal solutions. Thus, it is a logical consequence of our analysis to include the gained knowledge into the construction of new mutation operators. 
Therefore, the following section introduces a mutation concept based on sub-graph selection, single objective optimization, and re-implantation into the overall graph structure.

%%%%%%%%%%%%%%%%%%%%%%%%%%%%%%
%%%%% MUTATION OPERATORS %%%%%
%%%%%%%%%%%%%%%%%%%%%%%%%%%%%%

\section{Sub-Graph Mutation Operators}
\label{sec:mutation}

In this section, inspired by the valuable insights gained from the detailed analysis of Pareto-optimal solutions in the previous section, we give detailed descriptions of several sub-graph mutation operators. Moreover, we provide time complexity bounds and discuss important properties. 
The proposed mutation operators are based on a \emph{direct encoding}, i.e., a spanning tree $T$ is represented by its edge set $E_T$ of cardinality $n-1$.
It should be noted, that in our implementations -- for the sake of efficiency -- all graphs, sub-graphs and (spanning) trees are encoded using adjacency lists. Each adjacency list is realised as a linked-list. This allows the insertion, deletion and look-up of edges in time $O(\deg(v)) = O(n)$ where $\deg(v)$ is the \emph{degree} of node $v$, i.e., the number of nodes adjacent to~$v$. Edge-insertion takes only $\Theta(1)$ in the special case when we are sure that the edge does not yet exist in the graph. Note that by replacing the linked-list with a balanced search-tree, e.g., a red-black tree, we could realise all mentioned operations in guaranteed time $O(\log(\deg(v))) = O(\log(n))$. However, this would have no major effect on the runtime since it does not improve the algorithms' bottleneck-operations.

The direct encoding representation is also used by~\citet{KC01} in their work on the moMST. They use a straightforward, plain and simple mutation denoted \emph{one-edge-exchange mutation} (1EX). Here, an edge $e \in E$ is selected uniformly at random. If $e \in E(T)$ there is nothing to do and the mutation has no effect.\footnote{An effect-less mutation -- in our setting with complete graphs -- happens with a low probability $\frac{n-1}{\binom{n}{2}} = \frac{2}{n}$.} Otherwise, if $e \notin E(T)$, the edge-insertion yields a graph with edge set $E(T) \cup \{e\}$ which contains exactly one cycle. Next, the operator drops a random edge $f \in E(T)$ from this cycle which yields another spanning tree with edge set $(E(T) \cup \{e\}) \setminus \{f\}$. It runs in time $O(n)$ as locating a cycle requires a single depth first search and the cycle length is bounded by $n$.\footnote{To achieve this upper bound during initialization of the evolutionary algorithm a helper array with references to the edges of the graph has to be initialized in order to sample edges in constant time. Note however, that this has to be done only once in time $\Theta(m)$ before the optimisation process starts.}
\citet{BNG2019} biased the 1EX operator by putting stronger selection probabilities on \enquote{cheap} edges, i.e., edges with lower ranks according to a partial ordering of all edges. This approach was justified by an extensive empirical study on the relationship between edges' ranks and their probability of being part of Pareto-optimal spanning trees. In their work the best performing strategy was to set probabilities proportional to their domination count, i.e. the number of edges they are dominated by. We use this version for comparison and denote it as 1BEX in the following. Note that 1BEX requires a preprocessing of the edges once and runs in time $O(n)$, too.

1BEX mutation is a fine example for introducing bias into an evolutionary operator. The operators discussed below are hybrids of random mutation and local-search introducing even higher bias towards low-ranked sub-trees. The general idea in a nutshell: given a spanning tree $T$, drop a set of randomly selected (un)connected edges $E_s$ and re-establish the spanning tree property by applying a well-known single-objective MST algorithm to join the connected components of $E(T)\setminus E_s$ with focus on a single objective.
The pseudo-codes in the upcoming sections are formulated with readability in mind. I.e., we omit potentially non-obvious details that are necessary for an efficient implementation. In Appendix~\ref{sec:implementation_details} we provide much more detailed pseudo-codes.

% FIXME: According to Herberts answer the algorithm can be used if floating is deacitivated via H option
% See: https://tex.stackexchange.com/questions/219816/algorithm-in-ieee-format
\begin{algorithm}[t]
\caption{Sub-Graph Mutation (SG)}
\label{alg:sg}
\begin{algorithmic}[1]
\Require{Graph $G = (V, E, c = (c_1, c_2)^{\top})$, spanning tree $T$ of $G$,\newline{} threshold $\sigma \in \{3, \ldots, |V|\}$, round $\in \{$true, false$\}$}
    \State $v$ $\gets$ Random node from $V$
    \State $s$ $\gets$ Random integer from $\{3, \ldots, \sigma\}$
    \State $V_s$ $\gets$ \Call{BFS}{$T, v, s$} \Comment{Stop once BFS tree contains $s$ nodes}
    \State $\lambda$ $\gets$ Random weight between $0$ and $1$
    \If{round}
        \State $\lambda$ $\gets$ \Call{round}{$\lambda$} \Comment{Special case as introduced in~\cite{BG2017ParetoBeneficial}}
    \EndIf
    \State $G'$ $\gets$ $G[V_s] = (V_s, E_s, \lambda c_1 + (1 - \lambda)c_2)$ \Comment{$E_s$ contains all edges between nodes in $V_s$}
    \State $E_s^{*}$ $\gets$ \Call{Kruskal}{$G'$}
    \State \Return{$(E(T) \setminus E_s) \cup E_s^{*}$}
\end{algorithmic}
\end{algorithm}

\subsection{Sub-graph mutation}

Algorithm~\ref{alg:sg} gives a detailed outline of the \emph{sub-graph mutation} (SG) considering the bi-objective case. Provided an input graph $G = (V, E, c)$ with cost function $c : E \to \mathbb{R}^2_{>0}$ and a spanning tree $T$, the algorithm proceeds as follows. First, a random start node is sampled in line~1 alongside a random integer threshold $s$ of size at most $\sigma$ in line~2.

This threshold $s$ is used in line~3 to limit a Breadth-First-Search (BFS) sub-routine on the spanning tree $T$, i.e., BFS terminates once $s$ nodes, gathered in the set $V_s$, have been visited. Next, a random real-valued weight $\lambda$ between zero and one is sampled, rounded to the nearest integer value if the corresponding flag is set (lines~4~to~6). The $\lambda$ value is used to calculate new weights in the induced sub-graph $G' = G[V_s] = (V_s, E_s)$, which contains all edges between pairs of nodes in $V_s$ that also occur in $G$. $G'$ is then used as input for Kruskal's algorithm in line~8. The output of Kruskal's algorithm is a spanning tree on $G'$ and its edges $E_s^{*}$ replace the edges in $T$ connecting the nodes in $V_s$. The working principle is illustrated in Fig.~\ref{fig:sg_detailed} by example. Note that in a proper implementation, the graphs and trees are managed by means of adjacency lists. 

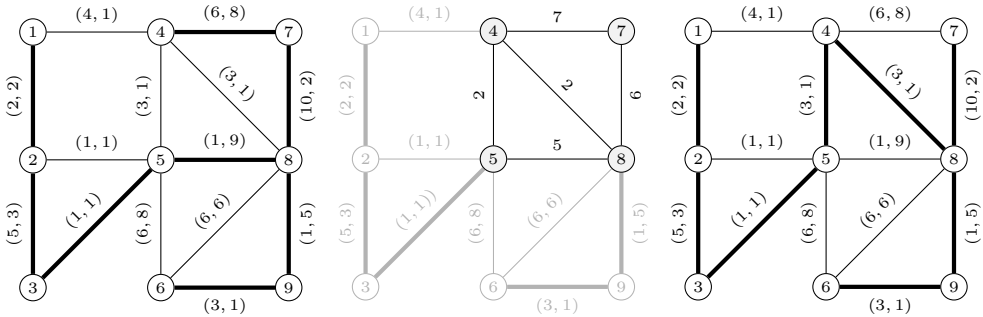
\begin{figure}[htbp]
  \begin{minipage}[c]{0.33\columnwidth}
  %\hspace*{-0.45cm}s
    \begin{tikzpicture}[scale=1.7]
      % nodes
      \begin{scope}[every node/.style={draw, circle, minimum size=1em, inner sep=0.03cm, font=\tiny}] %,
        \node (v0) at (0, 0) {$1$};%{$v_0$};
        \node (v1) at (0, -1) {$2$};%{$v_1$};
        \node (v2) at (0, -2) {$3$};%{$v_2$};
        \node (v3) at (1, 0) {$4$};%{$v_3$};
        \node (v4) at (1, -1) {$5$};%{$v_4$};
        \node (v5) at (1, -2) {$6$};%{$v_5$};
        \node (v6) at (2, 0) {$7$};%{$v_6$};
        \node (v7) at (2, -1) {$8$};%{$v_7$};
        \node (v8) at (2, -2) {$9$};%{$v_8$};
      \end{scope}

      % edges: s-t-(sub)path edges are thick
      \begin{scope}[every node/.style={font=\tiny}]
        \draw (v1) edge[-, ultra thick] node[above, sloped] {$(2,2)$} (v0);
        \draw (v0) edge[-] node[above, sloped] {$(4,1)$} (v3);

        \draw (v2) edge[-, ultra thick] node[above, sloped] {$(5,3)$} (v1);
        \draw (v1) edge[-] node[above, sloped] {$(1,1)$} (v4);

        \draw (v2) edge[-, ultra thick] node[above, sloped] {$(1,1)$} (v4);

        \draw (v3) edge[-] node[above, sloped] {$(3,1)$} (v4);
        \draw (v3) edge[-, ultra thick] node[above, sloped] {$(6,8)$} (v6);
        \draw (v3) edge[-] node[above, sloped] {$(3,1)$} (v7);

        \draw (v4) edge[-] node[above, sloped] {$(6,8)$} (v5);
        \draw (v4) edge[-, ultra thick] node[above] {$(1,9)$} (v7);

        \draw (v5) edge[-] node[above, sloped] {$(6,6)$} (v7);
        \draw (v5) edge[-, ultra thick] node[below, sloped] {$(3,1)$} (v8);

        \draw (v6) edge[-, ultra thick] node[below, sloped] {$(10,2)$} (v7);

        \draw (v7) edge[-, ultra thick] node[below, sloped] {$(1,5)$} (v8);
      \end{scope}
    \end{tikzpicture}
  \end{minipage}%\hspace*{-0.25cm}%\hfill
  \begin{minipage}[c]{0.33\columnwidth}
    \begin{tikzpicture}[scale=1.7]
      % nodes
      \begin{scope}[every node/.style={draw, circle, minimum size=1em, inner sep=0.03cm, font=\tiny}] %,
        \node[opacity=.3] (v0) at (0, 0) {$1$};%{$v_0$};
        \node[opacity=.3] (v1) at (0, -1) {$2$};%{$v_1$};
        \node[opacity=.3] (v2) at (0, -2) {$3$};%{$v_2$};
        \node[fill=gray!10] (v3) at (1, 0) {$4$};%{$v_3$};
        \node[fill=gray!10] (v4) at (1, -1) {$5$};%{$v_4$};
        \node[opacity=.3] (v5) at (1, -2) {$6$};%{$v_5$};
        \node[fill=gray!10] (v6) at (2, 0) {$7$};%{$v_6$};
        \node[fill=gray!10] (v7) at (2, -1) {$8$};%{$v_7$};
        \node[opacity=.3] (v8) at (2, -2) {$9$};%{$v_8$};
      \end{scope}

      % edges: s-t-(sub)path edges are thick
      \begin{scope}[every node/.style={font=\tiny}]
        \draw[opacity=.3] (v1) edge[-, ultra thick] node[above, sloped] {$(2,2)$} (v0);
        \draw[opacity=.3] (v0) edge[-] node[above, sloped] {$(4,1)$} (v3);

        \draw[opacity=.3] (v2) edge[-, ultra thick] node[above, sloped] {$(5,3)$} (v1);
        \draw[opacity=.3] (v1) edge[-] node[above, sloped] {$(1,1)$} (v4);

        \draw[opacity=.3] (v2) edge[-, ultra thick] node[above, sloped] {$(1,1))$} (v4);

        \draw (v3) edge[-] node[above, sloped] {$2$} (v4);
        \draw (v3) edge[-] node[above, sloped] {$7$} (v6);
        \draw (v3) edge[-] node[above, sloped] {$2$} (v7);

        \draw[opacity=.3] (v4) edge[-] node[above, sloped] {$(6,8)$} (v5);
        \draw (v4) edge[-] node[above] {$5$} (v7);

        \draw[opacity=.3] (v5) edge[-] node[above, sloped] {$(6,6)$} (v7);
        \draw[opacity=.3] (v5) edge[-, ultra thick] node[below, sloped] {$(3,1)$} (v8);

        \draw (v6) edge[-] node[below, sloped] {$6$} (v7);

        \draw[opacity=.3] (v7) edge[-, ultra thick] node[below, sloped] {$(1,5)$} (v8);
      \end{scope}
    \end{tikzpicture}
  \end{minipage}%\hfill
  \begin{minipage}[c]{0.33\columnwidth}
    \begin{tikzpicture}[scale=1.7]
      % nodes
      \begin{scope}[every node/.style={draw, circle, minimum size=1em, inner sep=0.03cm, font=\tiny}] %,
        \node (v0) at (0, 0) {$1$};%{$v_0$};
        \node (v1) at (0, -1) {$2$};%{$v_1$};
        \node (v2) at (0, -2) {$3$};%{$v_2$};
        \node (v3) at (1, 0) {$4$};%{$v_3$};
        \node (v4) at (1, -1) {$5$};%{$v_4$};
        \node (v5) at (1, -2) {$6$};%{$v_5$};
        \node (v6) at (2, 0) {$7$};%{$v_6$};
        \node (v7) at (2, -1) {$8$};%{$v_7$};
        \node (v8) at (2, -2) {$9$};%{$v_8$};
      \end{scope}

      % edges: s-t-(sub)path edges are thick
      \begin{scope}[every node/.style={font=\tiny}]
        \draw (v1) edge[-, ultra thick] node[above, sloped] {$(2,2)$} (v0);
        \draw (v0) edge[-] node[above, sloped] {$(4,1)$} (v3);

        \draw (v2) edge[-, ultra thick] node[above, sloped] {$(5,3)$} (v1);
        \draw (v1) edge[-] node[above, sloped] {$(1,1)$} (v4);

        \draw (v2) edge[-, ultra thick] node[above, sloped] {$(1,1)$} (v4);

        \draw (v3) edge[-, ultra thick] node[above, sloped] {$(3,1)$} (v4);
        \draw (v3) edge[-] node[above, sloped] {$(6,8)$} (v6);
        \draw (v3) edge[-, ultra thick] node[above, sloped] {$(3,1)$} (v7);

        \draw (v4) edge[-] node[above, sloped] {$(6,8)$} (v5);
        \draw (v4) edge[-] node[above] {$(1,9)$} (v7);

        \draw (v5) edge[-] node[above, sloped] {$(6,6)$} (v7);
        \draw (v5) edge[-, ultra thick] node[below, sloped] {$(3,1)$} (v8);

        \draw (v6) edge[-, ultra thick] node[below, sloped] {$(10,2)$} (v7);

        \draw (v7) edge[-, ultra thick] node[below, sloped] {$(1,5)$} (v8);
      \end{scope}
    \end{tikzpicture}
  \end{minipage}
  \caption{Working principle of the SGS mutation for $s = 4, \lambda = 0.5$. Left: source graph $G$ with bold spanning tree edges $E_T$. Center: BFS selected the node set $V_s = \{4, 5, 7, 8\}$ (light gray nodes) and build the induced sub-graph $G'=(V_s, E_s)$ with edge weights $0.5c_1(e) + 0.5c_2(e)$ for all edges of $e \in E_s$. All nodes and edges which do not belong to the induced sub-graph $G'$ are shown with reduced opacity. Right: finally, the operator computes an MST on $G'$ and creates a new spanning tree.}
  \label{fig:sg_detailed}
\end{figure}
Algorithm~\ref{alg:sg} is a modification/generalisation of the operator introduced in \citet{BG2017ParetoBeneficial}: (1) In our original proposal $s$ was a \enquote{soft} threshold, i.e., the sampled node set $V_s$ could have more than $s$ nodes, even up to $n$.\footnote{Extensive preliminary experiments revealed that this modification did not affect the performance in a significant way.} We introduced the modification to be able to control the runtime complexity, which depends on $\sigma$ (see below). (2) We allow $\lambda \in [0,1]$ while the original algorithm sampled $\lambda \in \{0, 1\}$. Note that this is a straight-forward generalisation and we expect better performance due to stronger flexibility. To discriminate the approaches by names, we keep SG as name for the vanilla \emph{sub-graph} mutation (parameter \verb|round=true|) and denote the \emph{scalarized} mutation with $\lambda \in [0,1]$ as SGS for \emph{sub-graph scalarized} (parameter \verb|round=false|).

%\todo{Need to either drop the following (since reviewer~2 complains) or mention that it is harder than we claim it at the moment to adapt to $q>2$.}\christian{we can drop it ... but I do not understand the reviewers comment anymore ... maybe it is not worth starting a discussion here - maybe in a later paper}\jakob{I agree.}
% \removed{It should be noted that Algorithm~\ref{alg:sg} can be extended to more than two objectives ($q \geq 3$) with only minor modifications. Actually, only the sampling of $\lambda$ needs adaptation in order to obtain uniformly distributed values $\lambda_1, \lambda_2, \ldots, \lambda_q$ with $\sum_{i=1}^{q} \lambda_i = 1$; all other components work out of the box. We believe that barycentric coordinates in a simplex may provide an adequate weighting%~\citep{grimme2015}, 
% but leave this research direction open for future investigations.}

\begin{theorem}
\label{thm:runtime_sg}
The SG mutation (see Algorithm~\ref{alg:sg}) has a worst case runtime complexity of
\begin{align*}
    \Theta(\sigma) + O(\sigma\Delta) + \Theta(|E_{\sigma}|\log \sigma) + \Theta(n)
\end{align*}
% \begin{align*}
%     O\left(\max\left\{\sigma \cdot \Delta, |E_{\sigma}| \cdot \log (|V_{\sigma}|), n\right\}\right)
% \end{align*}
where $\Delta$ is the maximum node degree in $G$. 
\end{theorem}
\begin{proof}
We consider the case that $s = \sigma$ (see line~2 in Algorithm~\ref{alg:sg}) and therefore work with $\sigma$ in the following. Furthermore, we assume that sampling of integer random numbers can be done in constant time. Hence, the first two lines in Algorithm~\ref{alg:sg} are negligible. 
Next, we analyse the restricted breath-first search (Algorithm~\ref{alg:sg}, line 3). BFS maintains an array \verb+visited+ of length $n$ to mark already visited nodes. We initialise this array once in time $\Theta(n)$ with all zero entries and keep it in memory (see implementation details in Section~\ref{sec:implementation_details}). The BFS call maintains a counter of visited nodes and a linked-list of visited nodes $V_{\sigma}$. Once the counter hits $\sigma$ BFS terminates immediately. Hence, the running time of the restricted BFS is $\Theta(|V_{\sigma}|) = \Theta(\sigma)$ since we are operating on a spanning-tree. 
Lines~4 to~6 only contribute a constant amount of time.
The next costly operation is building the induced sub-graph $G' = G[V_{\sigma}]$ in line~7. This requires to traverse the adjacency lists of all $\sigma$ nodes in $V_\sigma$ in $G$. We can access the adjacency lists in $G$ directly for each node $v \in V_{\sigma}$. In sum this takes time
\begin{align*}
    \sum_{v \in V_{\sigma}} \underbrace{\deg_G(v_s)}_{\leq \Delta} \leq \sigma\Delta.
\end{align*}
For each traversed edge $e = \{i,j\}$ we can check \verb+visited+$[i]$ and \verb+visited+$[j]$ in constant time and add $e$ to $G' = G[V_{\sigma}]$ if both entries are~1; this step comprises the generation of the scalarized cost function $\lambda c_1 + (1-\lambda)c_2$. Thus the running time is bounded by $O(\sigma \Delta)$.
The running time of Kruskal's algorithm~(line~8) is $\Theta(|E_{\sigma}| \log \sigma)$ when adopting the Union-Find data-structure for set management~\citep{tarjan1975SetUnion}.
The final step~(line~9) takes $\Theta(n)$ steps: we (i) create an empty graph with $n$ nodes in time $\Theta(n)$, (ii) copy all edges from $E_{\sigma}^{*}$ in time $\Theta(\sigma)$ and copy all edges from $E_T$ which do not introduce cycles (this can be achieved in $\Theta(n)$ since for each edge $e = \{i,j\} \in E_T$ be can again use the \verb+visited+ array; if exactly one of the entries is~1 and the other is~0, we add the edge). Finally a clean-up step is required (not explicitly stated in the pseudo-code) to reset the $\sigma$ 1-entries in the utility array \verb+visited+ to~0. This can be done in time $\Theta(\sigma)$ by traversing $V_{\sigma}$.
Considering the main steps, the overall running time is the sum of its parts
\begin{align*}
    \Theta(\sigma) + O(\sigma\Delta) + \Theta(|E_{\sigma}|\log \sigma) + \Theta(n)
\end{align*}
as claimed.
\end{proof}
The dominating summand depends on three major aspects: (i) the choice of the parameter $\sigma \in \{3, \ldots, n\}$. E.g., if $\sigma=\Theta(n)$ and the graph is dense with $\Theta(n^2)$ edges, the running time is clearly dominated by Kruskal's algorithm. (ii) the result of the sampling $s \in \{3, \ldots, \sigma\}$ and (iii) the density of the input graph $G$ or more specifically, the density of the sub-graphs selected by the restricted BFS call. For a better understanding we give the following results for complete graphs, $s = \sigma$ and a specific choice of $\sigma$.

\begin{theorem}
\label{thm:runtime_sg_finegrained}
The worst-case running time of SG mutation (see Algorithm~\ref{alg:sg}) with $s = \sigma$ on a complete graph with $n$ nodes is
\begin{itemize}
    \item $\Theta(m \log n) = \Theta(n^2\log n)$ if $\sigma = \Theta(n)$,
    \item $\Theta(n^{3/2})$ if $\sigma = \Theta(\sqrt{n})$,
    \item $\Theta(n \log n)$ if $\sigma = \Theta(\log n)$.
\end{itemize}
\end{theorem}
\begin{proof}
If the input graph $G=(V,E)$ is complete, it holds that $|E| = \Theta(n^2)$, $\Delta = n-1 = \Theta(n)$ and $|E_{\sigma}| = \Theta(\sigma^2)$. The results follow from plugging in the respective bounds into Theorem~\ref{thm:runtime_sg}.
\end{proof}
Setting $\sigma = o(n)$, e.g. $\sigma = \log(n)$, in combination with a sparse graph or a $k$-regular graph reduces the worst-case considerably. Unsurprisingly, in any case, the running time is lower bounded by $\Omega(n)$ since a copy of the input tree needs to be created.

% FIXME: According to Herberts answer the algorithm can be used if floating is deacitivated via H option
% See: https://tex.stackexchange.com/questions/219816/algorithm-in-ieee-format
\begin{algorithm}[H]
\caption{Unconnected Sub-graph Mutation (USG)}
\label{alg:usg}
\begin{algorithmic}[1]
    \Require{Graph $G = (V, E, c = (c_1, c_2)^{\top})$, spanning tree $T$ of $G$,\newline{} threshold $\sigma \in \{1, \ldots, |V|-1\}$, round $\in \{$true, false$\}$}
    \State $s$ $\gets$ Random integer from $\{3, \ldots, \sigma\}$
    \State $E_s$ $\gets$ Random subset of $\sigma$ edges from $E(T)$
    \State $T'$ $\gets$ $(V, E(T) \setminus E_s)$
    \State $\lambda$ $\gets$ Random weight between $0$ and $1$
    \If{round}
        \State $\lambda$ $\gets$ \Call{round}{$\lambda$}
    \EndIf
    \State $G'$ $\gets$ $(V, E, \lambda c_1 + (1 - \lambda) c_2)$
    \State $E_T^{*}$ $\gets$ \Call{Kruskal}{$G', T'$} \Comment{Kruskal using $T'$ for initialization of already connected components.}
    \State \Return{$E_T^{*}$}
\end{algorithmic}
\end{algorithm}

\subsection{Unconnected sub-graph mutation}

The \emph{Unconnected Sub-Graph} mutation (USG) drops the approach of replacing connected sub-trees in favor of more flexibility. Algorithm~\ref{alg:usg} outlines the procedure. Here, the major difference is that an arbitrary subset $E_s \subset E_T$ of size $s \leq \sigma$ is dropped (lines 2-3). The sampling of weights is identical to the procedure described for the SG operator (lines 4-6). In contrast to SG though, now we cannot operate on a induced sub-graph of bounded size which is to be replaced. Instead, we need to reconnect the connected components of $T = (V, E_T \setminus E_s)$. This is accomplished by a modified implementation of Kruskal's algorithm where we fix the initial connected components to those explicitly given by $T$. Note, that in a vanilla implementation of Kruskal's algorithm there are $n$ components after initialisation. Fig.~\ref{fig:usg_detailed} illustrates the working principle by example.
\begin{figure}[tbp]
  \begin{minipage}[c]{0.328\columnwidth}
  %\hspace*{-0.45cm}
    \begin{tikzpicture}[scale=1.50]
      % nodes
      \begin{scope}[every node/.style={draw, circle, minimum size=1em, inner sep=0.03cm, font=\tiny}] %,
        \node (v0) at (0, 0) {$1$};
        \node (v1) at (0, -1) {$2$};%{$v_1$};
        \node (v2) at (0, -2) {$3$};%{$v_2$};
        \node (v3) at (1, 0) {$4$};%{$v_3$};
        \node (v4) at (1, -1) {$5$};%{$v_4$};
        \node (v5) at (1, -2) {$6$};%{$v_5$};
        \node (v6) at (2, 0) {$7$};%{$v_6$};
        \node (v7) at (2, -1) {$8$};%{$v_7$};
        \node (v8) at (2, -2) {$9$};%{$v_8$};
      \end{scope}

      % edges: s-t-(sub)path edges are thick
      \begin{scope}[every node/.style={font=\tiny}]
        \draw (v1) edge[-, ultra thick] node[above, sloped] {$(2,2)$} (v0);
        \draw (v0) edge[-] node[above, sloped] {$(4,1)$} (v3);

        \draw (v2) edge[-, ultra thick] node[above, sloped] {$(5,3)$} (v1);
        \draw (v1) edge[-] node[above, sloped] {$(1,1)$} (v4);

        \draw (v2) edge[-, ultra thick] node[above, sloped] {$(1,1)$} (v4);

        \draw (v3) edge[-] node[above, sloped] {$(3,1)$} (v4);
        \draw (v3) edge[-, ultra thick] node[above, sloped] {$(6,8)$} (v6);
        \draw (v3) edge[-] node[above, sloped] {$(3,1)$} (v7);

        \draw (v4) edge[-] node[above, sloped] {$(6,8)$} (v5);
        \draw (v4) edge[-, ultra thick] node[above] {$(1,9)$} (v7);

        \draw (v5) edge[-] node[above, sloped] {$(6,6)$} (v7);
        \draw (v5) edge[-, ultra thick] node[below, sloped] {$(3,1)$} (v8);

        \draw (v6) edge[-, ultra thick] node[below, sloped] {$(10,2)$} (v7);

        \draw (v7) edge[-, ultra thick] node[below, sloped] {$(1,5)$} (v8);
      \end{scope}
    \end{tikzpicture}
  \end{minipage}%\hspace*{-0.25cm}%\hfill
  \begin{minipage}[c]{0.328\columnwidth}
    \begin{tikzpicture}[scale=1.50]
      % nodes
      \begin{scope}[every node/.style={draw, circle, minimum size=1em, inner sep=0.03cm, font=\tiny}] %,
        \node (v0) at (0, 0) {$1$};
        \node (v1) at (0, -1) {$2$};%{$v_1$};
        \node (v2) at (0, -2) {$3$};%{$v_2$};
        \node (v3) at (1, 0) {$4$};%{$v_3$};
        \node (v4) at (1, -1) {$5$};%{$v_4$};
        \node (v5) at (1, -2) {$6$};%{$v_5$};
        \node (v6) at (2, 0) {$7$};%{$v_6$};
        \node (v7) at (2, -1) {$8$};%{$v_7$};
        \node (v8) at (2, -2) {$9$};%{$v_8$};
      \end{scope}

      % edges: s-t-(sub)path edges are thick
      \begin{scope}[every node/.style={font=\tiny}]
        \draw (v1) edge[-] node[above, sloped] {$2$} (v0);
        \draw (v0) edge[-] node[above, sloped] {$4$} (v3);

        \draw (v2) edge[-, ultra thick] node[above, sloped] {$5$} (v1);
        \draw (v1) edge[-] node[above, sloped] {$1$} (v4);

        \draw (v2) edge[-, ultra thick] node[above, sloped] {$1$} (v4);

        \draw (v3) edge[-] node[above, sloped] {$3$} (v4);
        \draw (v3) edge[-, ultra thick] node[above, sloped] {$6$} (v6);
        \draw (v3) edge[-] node[above, sloped] {$3$} (v7);

        \draw (v4) edge[-] node[above, sloped] {$6$} (v5);
        \draw (v4) edge[-, ultra thick] node[above] {$1$} (v7);

        \draw (v5) edge[-] node[above, sloped] {$6$} (v7);
        \draw (v5) edge[-, ultra thick] node[below, sloped] {$3$} (v8);

        \draw (v6) edge[-] node[below, sloped] {$10$} (v7);

        \draw (v7) edge[-, ultra thick] node[below, sloped] {$1$} (v8);
      \end{scope}
    \end{tikzpicture}
  \end{minipage}%\hfill
  \begin{minipage}[c]{0.328\columnwidth}
    \begin{tikzpicture}[scale=1.50]
      % nodes
      \begin{scope}[every node/.style={draw, circle, minimum size=1em, inner sep=0.03cm, font=\tiny}] %,
        \node (v0) at (0, 0) {$1$};
        \node (v1) at (0, -1) {$2$};%{$v_1$};
        \node (v2) at (0, -2) {$3$};%{$v_2$};
        \node (v3) at (1, 0) {$4$};%{$v_3$};
        \node (v4) at (1, -1) {$5$};%{$v_4$};
        \node (v5) at (1, -2) {$6$};%{$v_5$};
        \node (v6) at (2, 0) {$7$};%{$v_6$};
        \node (v7) at (2, -1) {$8$};%{$v_7$};
        \node (v8) at (2, -2) {$9$};%{$v_8$};
      \end{scope}

      % edges: s-t-(sub)path edges are thick
      \begin{scope}[every node/.style={font=\tiny}]
        \draw (v1) edge[-, ultra thick] node[above, sloped] {$(2,2)$} (v0);
        \draw (v0) edge[-] node[above, sloped] {$(4,1)$} (v3);

        \draw (v2) edge[-, ultra thick] node[above, sloped] {$(5,3)$} (v1);
        \draw (v1) edge[-] node[above, sloped] {$(1,1)$} (v4);

        \draw (v2) edge[-, ultra thick] node[above, sloped] {$(1,1)$} (v4);

        \draw (v3) edge[-, ultra thick] node[above, sloped] {$(3,1)$} (v4);
        \draw (v3) edge[-, ultra thick] node[above, sloped] {$(6,8)$} (v6);
        \draw (v3) edge[-] node[above, sloped] {$(3,1)$} (v7);

        \draw (v4) edge[-] node[above, sloped] {$(6,8)$} (v5);
        \draw (v4) edge[-, ultra thick] node[above] {$(1,9)$} (v7);

        \draw (v5) edge[-] node[above, sloped] {$(6,6)$} (v7);
        \draw (v5) edge[-, ultra thick] node[below, sloped] {$(3,1)$} (v8);

        \draw (v6) edge[-] node[below, sloped] {$(10,2)$} (v7);

        \draw (v7) edge[-, ultra thick] node[below, sloped] {$(1,5)$} (v8);
      \end{scope}
    \end{tikzpicture}
  \end{minipage}
  \caption{Working principle of USG mutation for sampled $\sigma = 2$ and $\lambda = 1$. Left: source graph $G$ with thick spanning tree edges $E_T$. Center: USG drops edges $E_s = \{\{1, 2\}, \{7, 8\}\}$ (not thick anymore) and we see $G' = (V, E, c = c_1)$. Right: finally, the operator reconnects the connected components of $G'$ and creates a new spanning tree by re-inserting edge $\{1, 2\}$ and the new edge $\{4, 5\}$ (edge $\{4, 8\}$ would be possible too depending on the order the edges are checked).}
  \label{fig:usg_detailed}
\end{figure}

The next result shows that we gain flexibility at the expense of runtime complexity.
\begin{theorem}\label{thm:runtime_usg}
The USG mutation (see Algorithm~\ref{alg:usg}) has worst case runtime complexity of $\Theta(m \log n)$.
\end{theorem}
\begin{proof}
A crucial observation is that we need to apply Kruskal's algorithm to a modified version---modified in terms of the weighted sum---of the input graph $G$ with $n$ nodes and $m=O(n^2)$ edges.
Kruskal's algorithm requires $\Theta(m \log m) = \Theta(m \log n)$ operations for sorting the edges with a comparison based sorting algorithm (e.g., heap-sort) plus -- by using a union-find data-structure for disjoint sets -- time $O(m\alpha(n))$ for building the spanning tree. %\christian{Nur für mich: Die Gleichheit $\Theta(m \log m) = \Theta(m \log n)$ war mir erst nicht ganz klar, wenn $m=O(n^2)$ ist. Liegt es an $\log(n) = 2\log(n)$ und der Definition von $\Theta$?}\jakob{Ja genau. Für jedes $k$ gilt $\log(n^k)=k\log(n)$ und damit für \underline{konstantes $k$} (und hier ist $k \leq 2$ ist $\Theta(\log(n^k)) = \Theta(\log(n))$.} 
Here, $\alpha(n)$ is the extremely slowly growing inverse Ackermann function which takes values less than or equal~5 for all reasonable graph sizes. Hence, regardless of the choice of parameter $\sigma$, the runtime is dominated by Kruskal, i.e. $\Theta(m \log n) + O(m \alpha(n)) = \Theta(m \log n)$, as stated.
\end{proof}

In fact, when a single edge is dropped the spanning tree property can be re-established in time $O(m \alpha(n))$ by an algorithm proposed by \citet{Tarjan1979}. However, in the case of dropping multiple -- up to $\sigma \leq n-1$ -- edges, to the best of our knowledge, there is no algorithm that joins the connected components faster than a traditional MST algorithm. 

In analogy to the sub-graph mutation, we introduce USG and USGS as notation to discriminate the unweighted and scalarized versions of the approach, respectively.

\begin{figure}
    \centering
    \includegraphics[width=\textwidth]{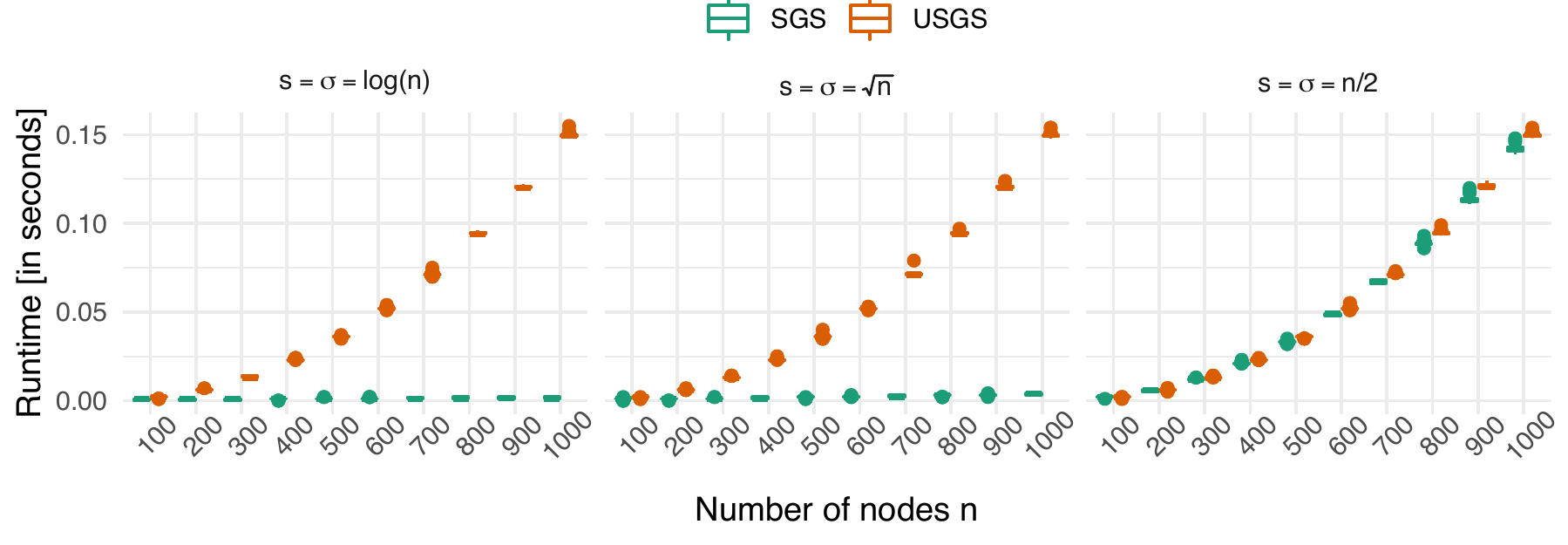}
    \caption{Distribution of runtimes (in seconds; $y$-axis) of SGS and USGS on complete graphs with $n$ nodes ($x$-axis). The plots are split by the choice of the $s$ parameter which is set to $\log(n)$ (left), $\sqrt{n}$ (center) and $n/2$ (right). The runtimes support the theoretical findings in Theorem~\ref{thm:runtime_sg_finegrained} and Theorem~\ref{thm:runtime_usg}.}
    \label{fig:runtimes_operators}
\end{figure}
The theoretical results are confirmed by experimental results in Fig.~\ref{fig:runtimes_operators}. Here, we plot the distribution of runtimes of SGS and USGS on complete graphs with $n \in \{100, 200, \ldots, 1\,000\}$ nodes varying the parameter $\sigma \in \{\log(n), \sqrt{n}, n/2\}$ (values rounded to the nearest integer) to match the regimes considered in Theorem~\ref{thm:runtime_sg_finegrained}.
The data was generated as follows: for each $n$ we created a complete graph of class~C1. Next, we sampled a random spanning tree of the graph and used it as input for both SGS and USGS setting $s = \sigma$ and logged the runtimes; this process was repeated 30 times independently.
We see that the theoretical observations are confirmed. While the runtimes of SGS increase with increasing $\sigma$, the runtimes of USGS are essentially unaffected. Moreover, we see very narrow box-plots indicating that lower and upper bounds indeed match.

\subsection{Potential algorithmic improvements}

We want to mention a couple of possible algorithmic improvements in the following. Note that in the SG mutation operator the single-objective MST algorithm can be easily replaced by an alternative. Prim's algorithm~\citep{Prim57} would reduce the worst-case running-time to $O(|E_\sigma| + \sigma \log \sigma)$ on dense graphs. 
% \removed{Implementing Karger's randomized MST algorithm~(Karger et al.m 1995) would even result in expected linear time $O(|E_\sigma|)$ (linear in the number of edges) with overwhelming probability~(Karger et al., 1995) at the cost of severe implementation overhead.}
The choice could even be made in an instance-specific way as Kruskal should be preferred over Prim on sparse graphs, say with $m = \Theta(n)$, or when edge weights are integers from a small domain such that the sorting of edges---the most expensive operation in Kruskal---can be realized in linear time (e.g., by adopting counting sort~\citep{cormen2009IntroductionToAlgorithms}). In contrast, Prim clearly outperforms Kruskal on dense graphs with $m = \Theta(n^2)$.
A legitimate question is why did we prefer Kruskal over Prim? The answer is two-fold: 1) Kruskal is easier to implement and 2) there is no alternative to Kruskal---or at least we are not aware of an alternative---for the unconnected USG mutation operator. This is due to the working principle of Prim's algorithm, which grows a spanning tree from a start node sequentially avoiding cycles. However, in USG, the algorithmic problem at hand is to reconnect existing connected components. Applying Kruskal seems to be the natural choice as it grows a forest until all components are connected and hence only the initialization of the union-find data structure requires algorithmic adaptation.

\subsection{Pareto-beneficial behaviour of the sub-graph mutation operators}

Let $T$ be a spanning tree and $T'$ be another spanning tree resulting from applying either sub-graph based mutation to $T$. We claim that $T$ cannot dominate $T'$. In other words: $T'$ either dominates $T$ or both trees are incomparable. This desirable property---termed \emph{Pareto-beneficial} property due to the fact that it drives convergence and diversity at the
same time without allowing deterioration---is illustrated in Fig.~\ref{fig:pareto_beneficial}.
Given a spanning tree $T$ SG drops a connected set $E_s \subset E(T)$ in favor of another subset $E_s^{*}$. Note that $E_s^{*}$ stems from an application of an optimal single-objective MST algorithm on a weighted sum scalarisation $c_{\lambda} = \lambda c_1 + (1 - \lambda)c_2$ of the the sub-graph induced by the nodes spanned by $E_s$. Thus, $E_s^{*}$ is (supported) efficient and hence $c(E_s^{*}) \preceq c(E_s)$ and as a direct consequence $c(T') = c(E(T) \setminus E_s \cup E_s^{*}) \preceq c(T)$. The proof for USG is omitted since it is very similar. We conclude the above statement in the following theorem.
\begin{theorem}
All proposed sub-graph based mutation operators are Pareto-beneficial.
\end{theorem}
Note that this property follows a local, greedy perspective, i.e., Pareto-beneficial behaviour holds only with respect to the input solution $T$. However, the mutant $T'$ may be dominated by other individuals in the population if a population-based EMOA is used. Moreover, sequential application of a Pareto-beneficial operator to $T$, i.e., mutating the mutant, may produce a tree which is dominated by $T$. Nevertheless, the property is desirable, since deterioration of solutions is avoided and mutants tend to head in the direction of the Pareto-front (see next section for examples).

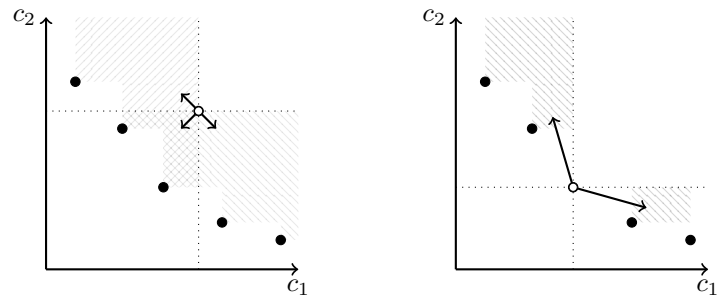
\begin{figure}[tbp]
  \centering
  %\begin{subfigure}[c]{0.49\columnwidth}
  \subfloat{
  \begin{tikzpicture}[scale=.78]
    \draw[->, thick] (0,0) -- (4.3,0) node[below] {$c_1$};
    \draw[->, thick] (0,0) -- (0, 4.3) node[left] {$c_2$};

    % draw Pareto-optimal points
    \path (0.5,3.2) coordinate (pf1);
    \path (1.3,2.4) coordinate (pf2);
    \path (2, 1.4) coordinate (pf3);
    \path (3, 0.8) coordinate (pf4);
    \path (4, 0.5) coordinate (pf5);

    % draw segmentation lines
    \draw[-, dotted] (0, 2.7) -- (4.3, 2.7);
    \draw[-, dotted] (2.6,0) -- (2.6, 4.3);

    % draw hatched area where f1 is better
    \fill[pattern=north east lines, opacity=0.3] (0.5, 4.3) -- (pf1) -- (1.3, 3.2) -- (pf2) -- (2, 2.4) -- (pf3) -- (2.6, 1.4) -- (2.6, 4.3) -- cycle;

    % draw hatched area where f2 is better
    \fill[pattern=north west lines, opacity=0.3] (1.3, 2.7) -- (pf2) -- (2, 2.4) --  (pf3) -- (3, 1.4) -- (pf4) -- (4, 0.8) -- (pf5) -- (4.3, 0.5) -- (4.3, 2.7) -- cycle;

    % draw some dominated solution
    \path (2.6, 2.7) coordinate (emoa);

    % draw arrows of possible directions
    \draw[->, thick] (emoa) -- (2.3, 2.4);
    \draw[->, thick] (emoa) -- (2.3, 3.0);
    \draw[->, thick] (emoa) -- (2.9, 2.4);

    \draw[fill=black] (emoa) circle (0.8mm);
    \draw[fill=white] (emoa) circle (0.7mm);

    % draw Pareto-front
    \foreach \ep in {pf1, pf2, pf3, pf4, pf5}
      \draw[fill=black] (\ep) circle (0.8mm);

  \end{tikzpicture}
  %\end{subfigure}
  }
  \hspace{1cm}
  \subfloat{
  %\begin{subfigure}[c]{0.49\columnwidth}
  \begin{tikzpicture}[scale=.78]
    \draw[->, thick] (0,0) -- (4.3,0) node[below] {$c_1$};
    \draw[->, thick] (0,0) -- (0, 4.3) node[left] {$c_2$};

    % draw Pareto-optimal points
    \path (0.5,3.2) coordinate (pf1);
    \path (1.3,2.4) coordinate (pf2);
    \path (2, 1.4) coordinate (pf3);
    \path (3, 0.8) coordinate (pf4);
    \path (4, 0.5) coordinate (pf5);

    % draw segmentation lines
    \draw[-, dotted] (0, 1.4) -- (4.3, 1.4);
    \draw[-, dotted] (2,0) -- (2, 4.3);

    % draw hatched area where f1 or f2 is better
    \fill[pattern=north west lines, opacity=0.4] (pf1) -- (1.3, 3.2) -- (pf2) -- (2, 2.4) -- (pf3) -- (3, 1.4) -- (pf4) -- (4,0.8) -- (pf5) -- (4, 1.4) -- (pf3) -- (2, 4.3) -- (0.5, 4.3) -- cycle;

    % draw arrows of possible directions
    \draw[->, thick] (pf3) -- ($(pf3) - (0.35,-1.2)$);
    \draw[->, thick] (pf3) -- ($(pf3) - (-1.25,0.35)$);

    % draw Pareto-front
    \foreach \ep in {pf1, pf2, pf4, pf5}
      \draw[fill=black] (\ep) circle (0.8mm);

    \draw[fill=black] (pf3) circle (0.8mm);
    \draw[fill=white] (pf3) circle (0.7mm);

  \end{tikzpicture}
  %\end{subfigure}
  }
  \caption{Illustration of the objective space regions reachable by application of a sub-graph mutation operator on a non Pareto-optimal spanning tree (left) and a Pareto-optimal spanning tree (right).}
  \label{fig:pareto_beneficial}
\end{figure}

\subsection{Random walks}

We close this section by considering the empirical behaviour of the proposed mutation operators by means of visualization. To do so we investigate \emph{random walks} of the mutation operators, i.e., we initialize a random spanning tree $T$ using Broder's algorithm~\citep{broder1989RandomSpanningTrees}, iteratively create a mutant $T'$ by applying one of the proposed mutation operators and replace $T$ with $T'$ \underline{without} dominance checks. Fig.~\ref{fig:random_walk_gallery} shows each three random walks of length 100 of SG and USG as well as the scalarized versions SGS and USGS with $\sigma = n/4 = 25$. The plots show the random walks in the objective space for each considered instance of classes C1 to C4 with $n = 100$ nodes by drawing arrows from $T$ to $T'$. We can clearly identify patterns. First of all, we can see the obvious Pareto-beneficial behaviour of all operators as $T'$ never dominates $T$ in the beginning of the random walks (except for class C3 where by design of the graph class itself the initial points are close to the extremely large Pareto-front). Secondly, while SG and USG seem to get stuck half-way with an oscillating behaviour for C1 and C2, the scalarized versions manage to advance faster and in a more directed manner towards the Pareto-front. These observations are consistent across all instances of the corresponding classes. Furthermore, a close look reveals that dropping unconnected edges seems indeed advantageous as USGS performs slightly better than its connected counterpart SGS; this effect is even stronger for C1.
In contrast, for classes C3 and C4, there seems to be no difference between SG(S) and USG(S) in terms of random walks. For C3 we observe an oscillating behaviour for all sub-graph operators which is in par with the large number of Pareto-optimal solutions and the structure of the edge 
\begin{figure}[h!]
    \centering
    \includegraphics[width=\textwidth]{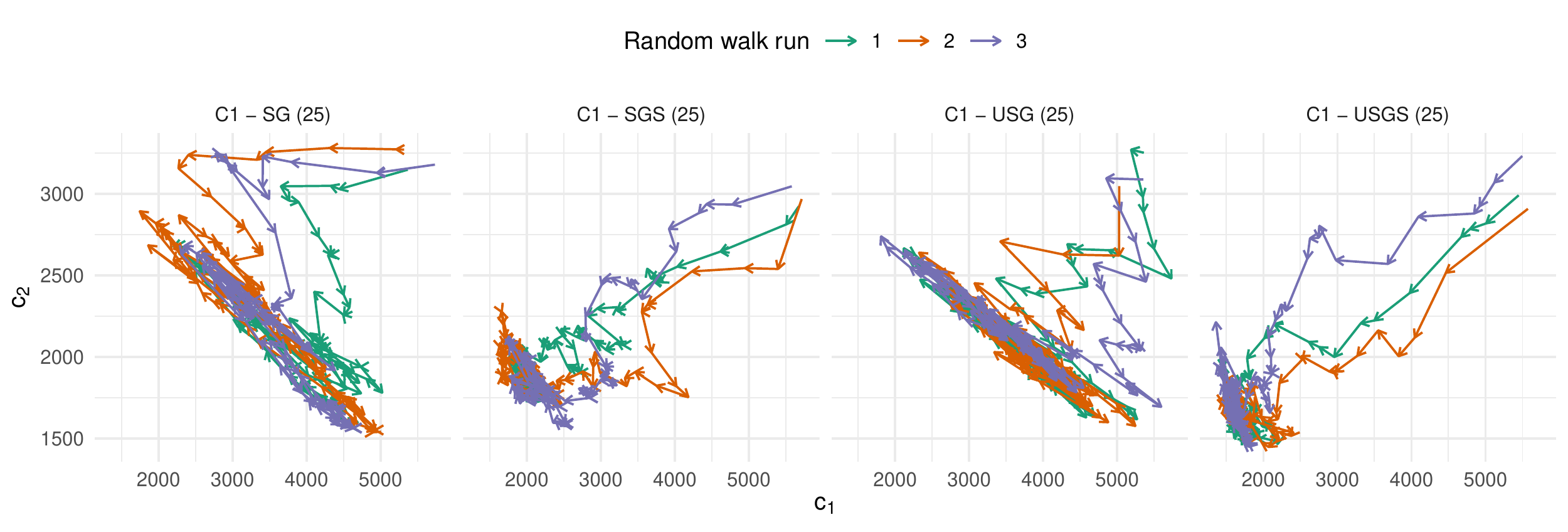}
    \includegraphics[width=\textwidth, trim = 0 7pt 0 32pt, clip]{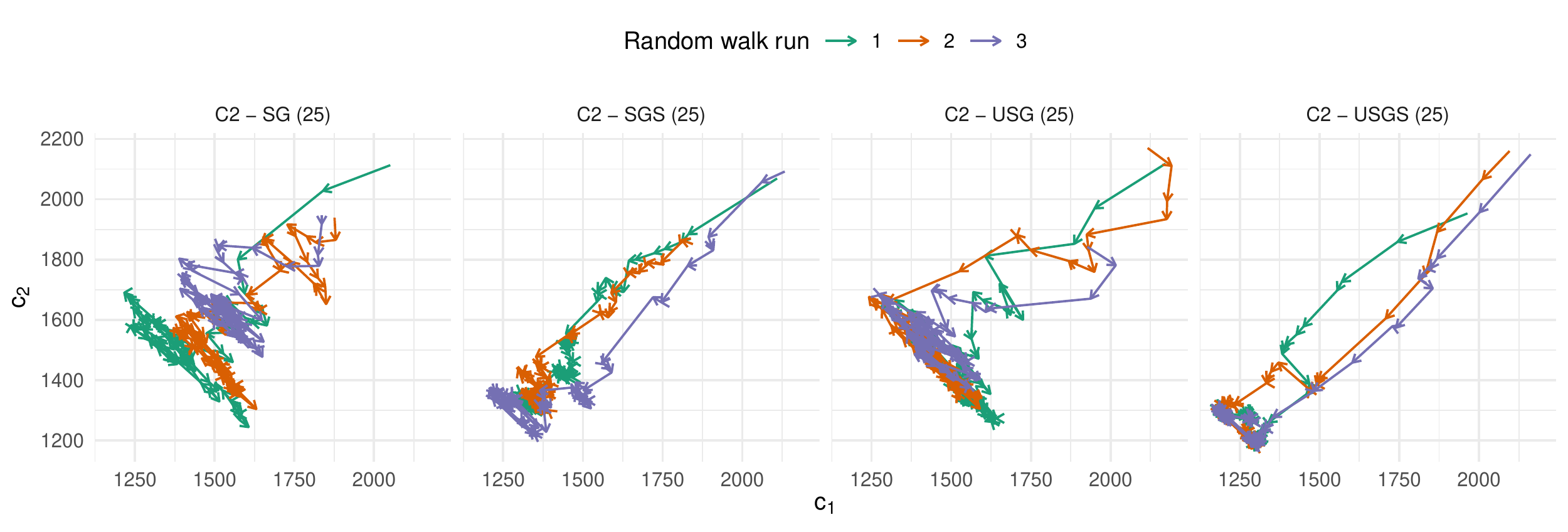}
    \includegraphics[width=\textwidth, trim = 0 7pt 0 32pt, clip]{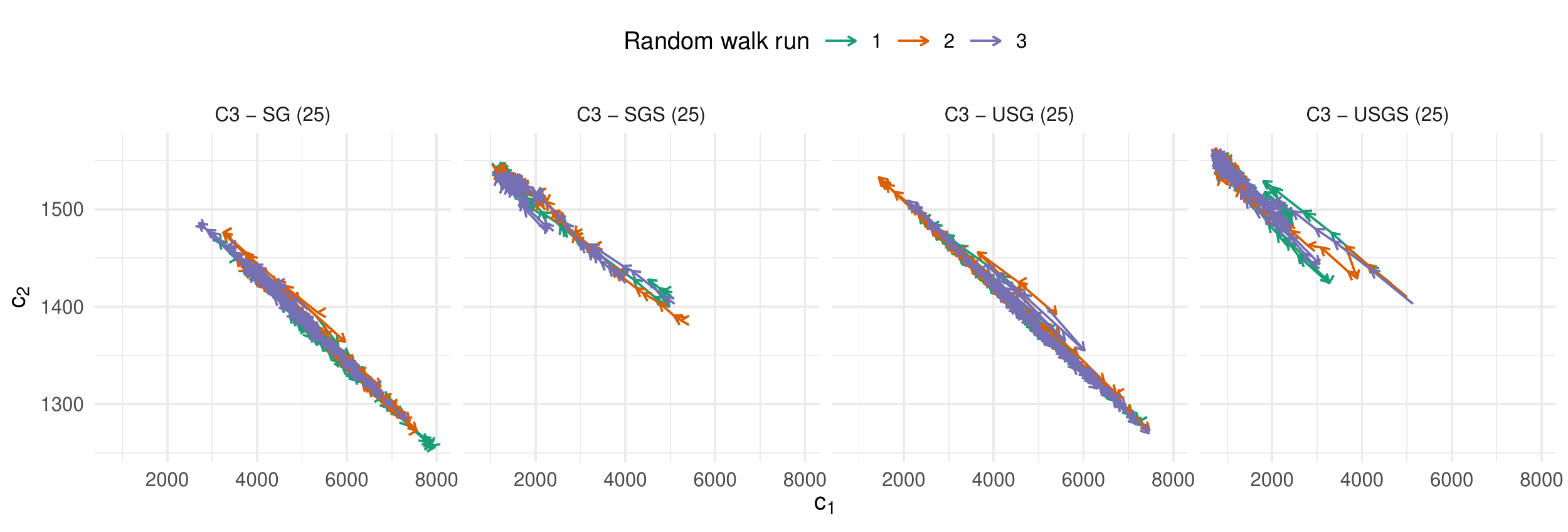}
    \includegraphics[width=\textwidth, trim = 0 7pt 0 32pt, clip]{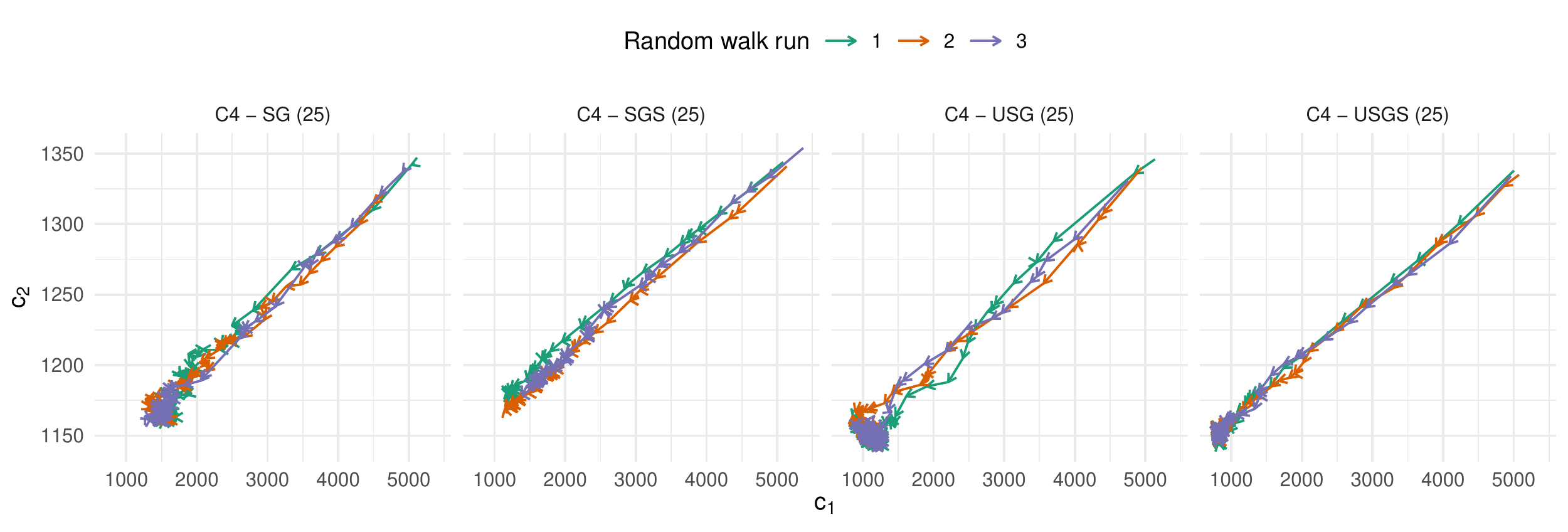}
    \caption{Visualization of each three random walks of the SG, SGS, USG and USGS mutation with $\sigma = 25$ on an instance $(n=100)$ of class C1 to C4 (top to bottom).}
    \label{fig:random_walk_gallery}
\end{figure}
weights (see 3rd column in Fig.~\ref{fig:exemplary_instances}). In the case of C4 we observe a rapid and strongly directed advance towards the front. Again, this can be explained by structural properties of the edge weights. Here, only a small fraction of edges are constant members of Pareto-optimal solutions (see 4th column in Fig.~\ref{fig:exemplary_instances}) and hence once added those edges are not removed anymore by subsequent mutations.
It seems likely that that these behavioral insights carry over to population-based EMOAs in which survival selection guides the evolutionary search process. We expect the scalarized mutation operators to perform better in the subsequent benchmark study.

%%%%%%%%%%%%%%%%%%%%%%%
%%%%% EXPERIMENTS %%%%%
%%%%%%%%%%%%%%%%%%%%%%%

\section{Experiments}
\label{sec:experiments}

In this chapter we conduct an extensive experimental study to compare the proposed sub-graph based mutation operators with operators from the literature. First, we detail the experimental setup and evaluate the results afterwards.

\subsection{Experimental setup}

For benchmarking we use instances of complete graphs with 25, 50, 100 and 250 nodes respectively.\footnote{We are aware that more sparse graphs, e.g. with $m=\Theta(n)$ are of interest as well. However, we decided for complete graphs due to (1) the maximum search space size of $n^{n-2}$ and (2) preliminary experiments that showed that the here presented results transfer to the sparse setting.} For each benchmark class introduced in Section~\ref{sec:instances} and each number of nodes we generated 10 instances by means of the R package \texttt{grapherator}~\citep{B2018grapherator}. In total, our benchmark set consists of $4~\text{ classes } \times 4~\text{ sizes } \times 10 = 160$ instances.

As mentioned earlier, we consider the widely used EMOA NSGA-II~\citep{DPAM02} as encapsulating meta-heuristic for eleven evaluated mutation operators. In a nutshell, NSGA-II uses a ranking of candidate solutions based on fast non-dominated sorting for convergence and crowding distance for diversity preservation.
%We consider the two widely used EMOAs NSGA-II~\citep{DPAM02} and SMS-EMOA~\citep{beume2007}. NSGA-II 
%uses a ranking of candidate solutions based on fast non-dominated sorting for convergence and crowding distance for diversity preservation. Contrary, SMS-EMOA follows a steady-state, indicator-based approach. Here, a single solution is generated in each generation and the solution with the smallest hypervolume contribution is dropped. These algorithms serve as encapsulating EMOAs for eleven evaluated mutation operators.
Note that we do not use crossover operators, because the focus of this study is the investigation of tailored mutation. Their effects may be blurred or even vanish by allowing interactions with other variation operators.
We use Zhou and Gen's uniform mutation~\citep{ZG99} on the Pr\"ufer-code representation (UNIFORM), unbiased one-edge-exchange (1EX) and its biased version (1BEX) on a direct encoding. Furthermore, we use the (scalarized) sub-graph mutation (SGS and SG) as well as (scalarized) unconnected sub-graph mutation (USGS and USG), respectively. We set $\sigma = n/2$ for all sub-graph based operators to allow for both small and large mutation effects. In addition, we consider SGSlog and USGSlog, versions where $\sigma = (\log n)^2$, i.e., much smaller.\footnote{This is not done for SG and USG since, the scalarized version showed the better results in preliminary experiments.} This is done to investigate whether we can achieve similar performance by replacing very small parts of the graphs. Note that this setting can lead to substantially decreased worst-case running times for SGSlog in comparison to SGS on sparse graphs, but does not affect the worst-case running time of USGSlog at all (recall the runtime results in Section~\ref{sec:mutation}).
As further baseline, we adopt the simple weighted sum approach~(WEIGHTED SUM) described in Section~\ref{sec:analysis}, which generates a subset of the supported efficient solutions of the Pareto-front by solving single-objective MST-problems for equidistantly spread weights $\lambda_k \in [0, 1], k = 1, \ldots, 50\,000$.
In addition we run Corley's algorithm~\citep{Corley1985}, a straight-forward exact extension of Prim's single-objective MST algorithm, once on every instance with a time limit of 48h and 16GB RAM. Corley terminated in only 13 out of 160 runs on small C4 instances only (all other runs were prematurely terminated due to excessive memory usage; this is unsurprising since the number of partial non-dominated solutions can grow exponentially) and those solutions were also discovered by the weighted-sum approach. As a consequence we do not consider Corley's algorithm in the following evaluation since its results are subsumed by WEIGHTED SUM.
We choose a population size of $100$ independent of the instance size $n$, perform $30$ independent runs of each EMOA/mutation combination on each problem instance for statistical soundness of subsequent evaluations and use $1000 \cdot n$ function evaluations as the only termination condition. Progress is monitored at $10\%, 50\%$ and $100\%$ of allowed total function evaluations. 

All algorithms including mutation operators are implemented in the R package \texttt{mcMST}~\citep{Bo17}.\footnote{\url{https://github.com/jakobbossek/mcMST}} 
Data and code can be found in a public GitHub repository (see Appendix~\ref{sec:experiments_details} for more details on the experiments).\footnote{\url{https://github.com/jakobbossek/ECJ-MOMST-Subgraph-Mutation}}

\begin{figure}[t]
    \centering
    \includegraphics[width=\columnwidth]{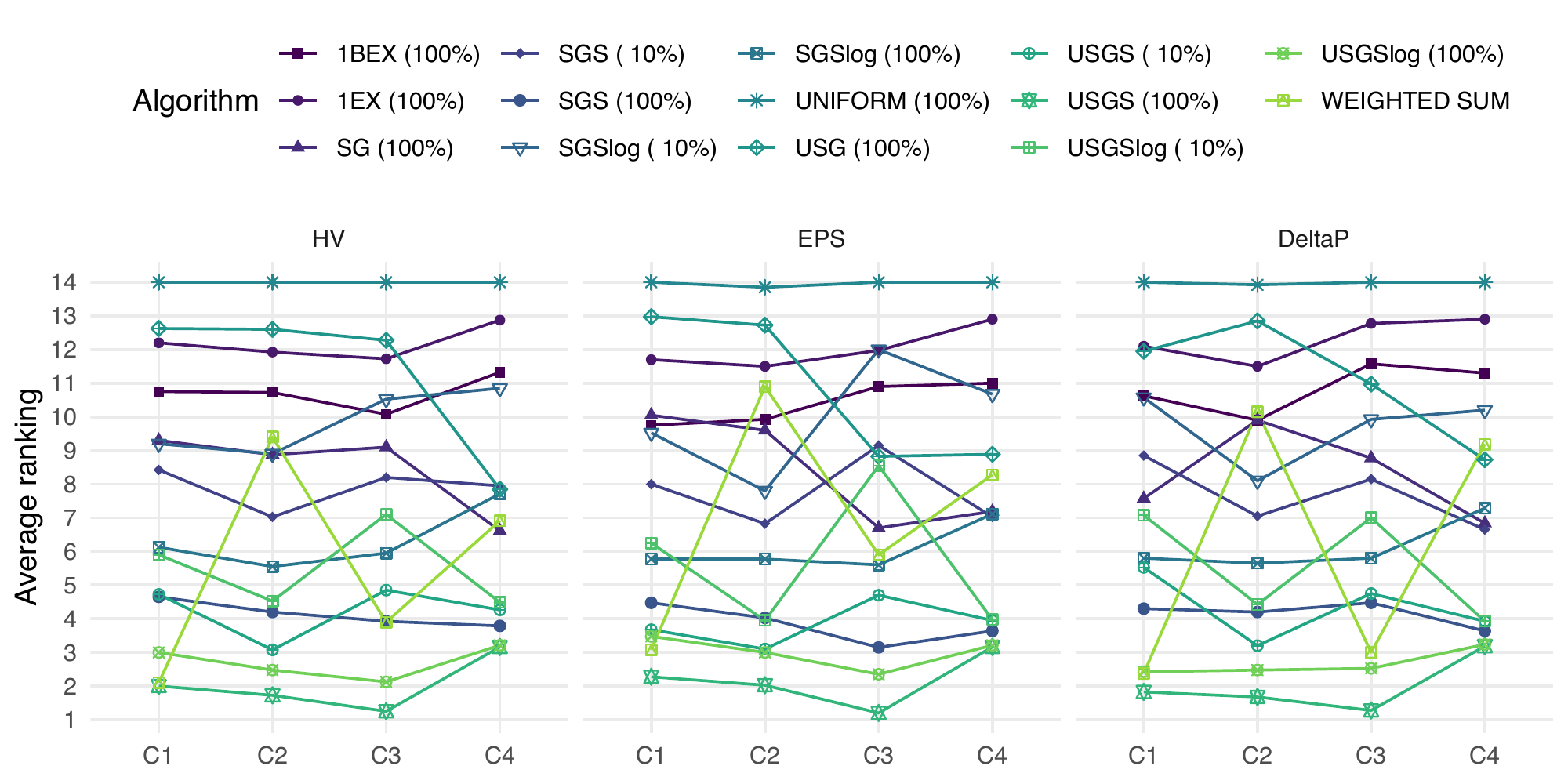}
    \caption{Mean rankings (lower is better) of indicators split by class and colored by algorithm/mutation. In the legend, the values in parentheses indicates the percentage of maximum function evaluations.}
    \label{fig:indicator_mean_ranks}
\end{figure}

\subsection{Experimental analysis}

%\jakob{Missing description of our baseline PRIM; moreover, rename to WS for weighted sum to not get confused with Corley's extension of Prim's algorithm.s}\christian{See comment in experimental setup.}
%\jakob{Christian to add some more details.}\christian{I have added some text to describe HV and $\varepsilon$ dominance in some more detail.} 
The comparison of algorithm performance
%\footnote{Note that the presented results hold for both wrapping EMOAs, NSGA-II and SMS-EMOA. Therefore and due to space restrictions, we only present the results for NSGA-II here.} 
is shown via a ranking in Fig.~\ref{fig:indicator_mean_ranks} Therein, the average ranking (the lower the better) of mutation operators is shown for all classes, all instances and all runs with respect to the hypervolume indicator~(HV), the $\varepsilon$-indicator~\citep{ZTLF03}, and $\Delta p$~\citep{SELC12}. The first two indicators predominantly measure convergence of a solution set towards the true Pareto-front. While the hypervolume indicator measures the overall volume enclosed between the non-dominated solutions and a reference set (possibly the Pareto-front) in objective space, epsilon dominance basically uses the minimum translation of a solution set in objective space such that the reference set is covered as measure for convergence. The $\Delta p$ measure focuses on the diversity and (even) distribution of solution sets. For each instance, the reference set is approximated by the non-dominated points in the union of all approximations obtained in all runs and all algorithms on the particular instance.

We observe that the UNIFORM mutation based on Pr\"ufer number encoding is inferior to all other operators. That is in line with earlier findings~\citep{Gottlieb2001,BG2017ParetoBeneficial}. For the simple edge exchange mutation (1EX) we can confirm our theoretical thoughts from Section~\ref{sec:analysis}: the deactivation of constant edges (edges that occur in all or most optimal solutions) deteriorates the performance of this operator. As expected, we find that the biased exchange mutation 1BEX performs better due to inclusion of dominance ranking information.

\begin{figure}[h!]
    \centering
    \includegraphics[width=0.9\columnwidth]{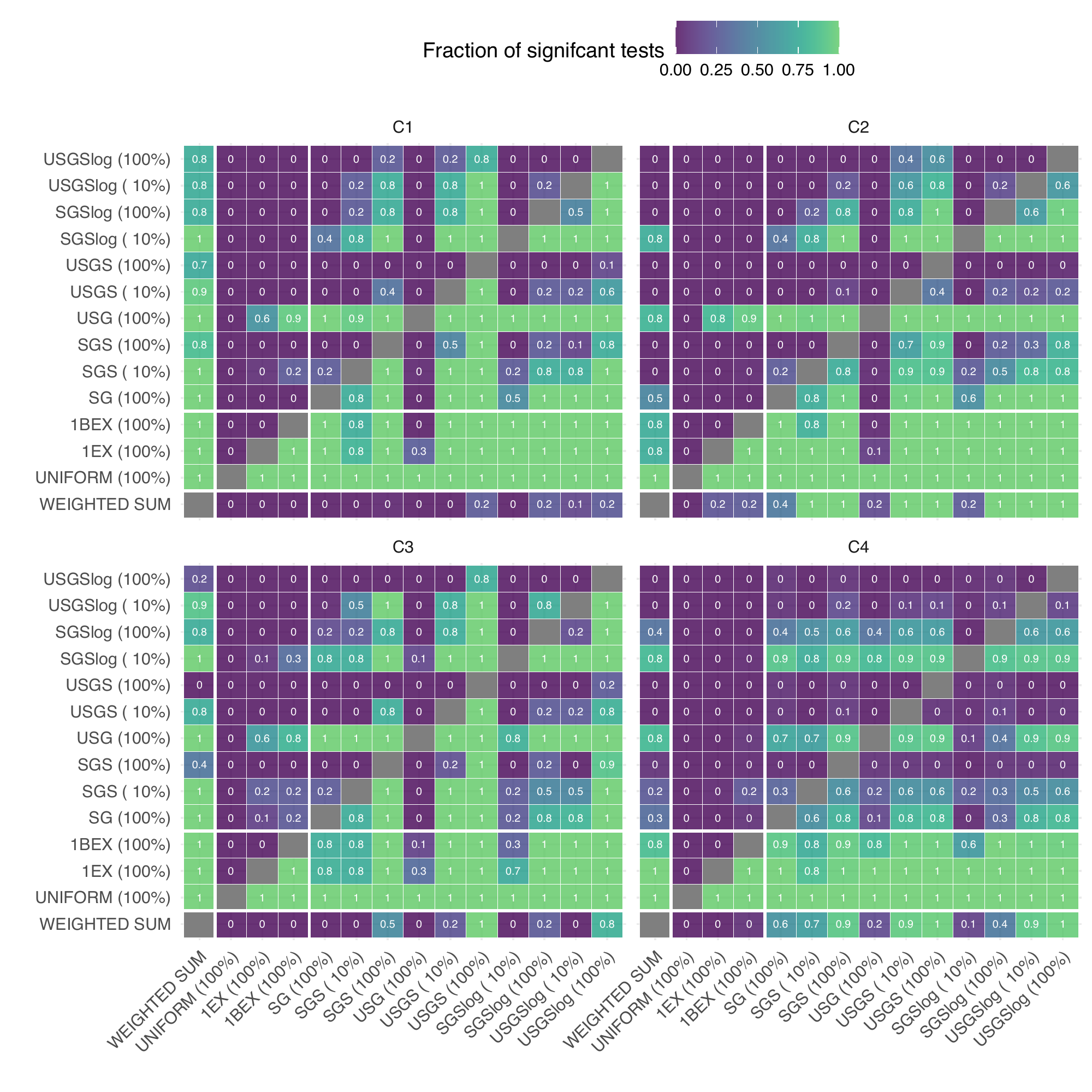}
    \caption{Visualisation of the fraction of pairwise Wilcoxon-Mann-Whitney tests to check whether the column-algorithm has a significantly $(\alpha = 0.01$) lower HV-indicator median value than the row-algorithm. Results are split by graph class. Bonferroni-Holm $p$-value adjustments was applied in order to avoid multiple-testing issues.}
    \label{fig:pairwise_tests_HV}
\end{figure}
Far better performance can, however, be observed for the proposed Pareto-beneficial scalarized operators SGS and USGS. While SGS exposes mediocre performance compared to all approaches, USGS outranks all others. Interestingly for USGS, the results have already good ranks after 10\% of total function evaluations. This compensates (up to a certain degree) for the larger runtime complexity of the operator. 
A detailed view on the pairwise comparison of operator performance in Fig.~\ref{fig:pairwise_tests_HV} shows that USGS~(10\%) and USGS~(100\%) significantly outperform all baseline algorithms, SG, USG, as well as SGS in many cases with respect to hypervolume indicator (see also Fig.~\ref{fig:pairwise_tests_eps_deltap}) for similar plots with respect to the $\varepsilon$-indicator and $\Delta p$).

These findings also hold for the even more runtime-optimized approaches USGSlog and SGSlog, in which the maximum of considered edges is reduced considerably. These operators rank almost as good as their more costly counterparts. In addition, SGS and USGS frequently outperform WEIGHTED-SUM. Note that the latter has a slight advantage in this comparison, since approximation sets of the EMOAs are limited in size by the population size of~100, while there may be way more distinct supported efficient solutions found by WEIGHTED SUM (recall that it is solved for $5\,000$ different weights).

\begin{landscape}
\begin{table}

\caption{\label{tab:indicators_n100}Results for instances with 100 nodes: Mean, standard deviation~(\textbf{sd}) and results of Wilcoxon-Mann-Whitney tests at significance level $\alpha=0.01$ (\textbf{stat}) with respect to HV-indicator and $\varepsilon$-indicator respectively. The \textbf{stat}-column is to be read as follows: a value $X^{+}$ indicates that the indicator for the column algorithm (note that algorithms are numbered and color-encoded in the second row) is significantly lower than the one of algorithm $X$. Lowest indicator values are highlighted in \textbf{bold-face}.}
\centering
\begin{tiny}
\renewcommand{\arraystretch}{1}
\renewcommand{\tabcolsep}{3.7pt}
\begin{tabular}[t]{rrrrrrrrrrrrrrrrrrr}
\toprule
\multicolumn{1}{c}{\textbf{ }} & \multicolumn{1}{c}{\textbf{ }} & \multicolumn{1}{c}{\textbf{ }} & \multicolumn{8}{c}{\textbf{HV-indicator}} & \multicolumn{8}{c}{\textbf{$\varepsilon$-indicator}} \\
\cmidrule(l{3pt}r{3pt}){4-11} \cmidrule(l{3pt}r{3pt}){12-19}
\multicolumn{1}{c}{\textbf{ }} & \multicolumn{1}{c}{\textbf{ }} & \multicolumn{1}{c}{\textbf{ }} & \multicolumn{2}{c}{\textbf{UNIF (\textcolor{violet}{1})}} & \multicolumn{2}{c}{\textbf{1BEX (\textcolor{brown}{2})}} & \multicolumn{2}{c}{\textbf{SGS (\textcolor{teal}{3})}} & \multicolumn{2}{c}{\textbf{USGS (\textcolor{purple}{4})}} & \multicolumn{2}{c}{\textbf{UNIF (\textcolor{violet}{1})}} & \multicolumn{2}{c}{\textbf{1BEX (\textcolor{brown}{2})}} & \multicolumn{2}{c}{\textbf{SGS (\textcolor{teal}{3})}} & \multicolumn{2}{c}{\textbf{USGS (\textcolor{purple}{4})}} \\
\cmidrule(l{3pt}r{3pt}){4-5} \cmidrule(l{3pt}r{3pt}){6-7} \cmidrule(l{3pt}r{3pt}){8-9} \cmidrule(l{3pt}r{3pt}){10-11} \cmidrule(l{3pt}r{3pt}){12-13} \cmidrule(l{3pt}r{3pt}){14-15} \cmidrule(l{3pt}r{3pt}){16-17} \cmidrule(l{3pt}r{3pt}){18-19}{}
 & $n$ & $I$ & \textbf{mean $\pm$ sd} & \textbf{stat} & \textbf{mean $\pm$ sd} & \textbf{stat} & \textbf{mean $\pm$ sd} & \textbf{stat} & \textbf{mean $\pm$ sd} & \textbf{stat} & \textbf{mean $\pm$ sd} & \textbf{stat} & \textbf{mean $\pm$ sd} & \textbf{stat} & \textbf{mean $\pm$ sd} & \textbf{stat} & \textbf{mean $\pm$ sd} & \textbf{stat}\\
\end{tabular}
\end{tiny}
\end{table}

\end{landscape}

\begin{landscape}
\begin{table}[p]

\caption{\label{tab:indicators_n250}Results for instances with 250 nodes: Mean, standard deviation~(\textbf{sd}) and results of Wilcoxon-Mann-Whitney tests at significance level $\alpha=0.01$ (\textbf{stat}) with respect to HV-indicator and $\varepsilon$-indicator respectively. The \textbf{stat}-column is to be read as follows: a value $X^{+}$ indicates that the indicator for the column algorithm (note that algorithms are numbered and color-encoded in the second row) is significantly lower than the one of algorithm $X$. Lowest indicator values are highlighted in \textbf{bold-face}.}
\centering
\begin{tiny}
\renewcommand{\arraystretch}{1}
\renewcommand{\tabcolsep}{3.7pt}
\begin{tabular}[t]{rrrrrrrrrrrrrrrrrrr}
\toprule
\multicolumn{1}{c}{\textbf{ }} & \multicolumn{1}{c}{\textbf{ }} & \multicolumn{1}{c}{\textbf{ }} & \multicolumn{8}{c}{\textbf{HV-indicator}} & \multicolumn{8}{c}{\textbf{$\varepsilon$-indicator}} \\
\cmidrule(l{3pt}r{3pt}){4-11} \cmidrule(l{3pt}r{3pt}){12-19}
\multicolumn{1}{c}{\textbf{ }} & \multicolumn{1}{c}{\textbf{ }} & \multicolumn{1}{c}{\textbf{ }} & \multicolumn{2}{c}{\textbf{UNIF (\textcolor{violet}{1})}} & \multicolumn{2}{c}{\textbf{1BEX (\textcolor{brown}{2})}} & \multicolumn{2}{c}{\textbf{SGS (\textcolor{teal}{3})}} & \multicolumn{2}{c}{\textbf{USGS (\textcolor{purple}{4})}} & \multicolumn{2}{c}{\textbf{UNIF (\textcolor{violet}{1})}} & \multicolumn{2}{c}{\textbf{1BEX (\textcolor{brown}{2})}} & \multicolumn{2}{c}{\textbf{SGS (\textcolor{teal}{3})}} & \multicolumn{2}{c}{\textbf{USGS (\textcolor{purple}{4})}} \\
\cmidrule(l{3pt}r{3pt}){4-5} \cmidrule(l{3pt}r{3pt}){6-7} \cmidrule(l{3pt}r{3pt}){8-9} \cmidrule(l{3pt}r{3pt}){10-11} \cmidrule(l{3pt}r{3pt}){12-13} \cmidrule(l{3pt}r{3pt}){14-15} \cmidrule(l{3pt}r{3pt}){16-17} \cmidrule(l{3pt}r{3pt}){18-19}{}
 & $n$ & $I$ & \textbf{mean $\pm$ sd} & \textbf{stat} & \textbf{mean $\pm$ sd} & \textbf{stat} & \textbf{mean $\pm$ sd} & \textbf{stat} & \textbf{mean $\pm$ sd} & \textbf{stat} & \textbf{mean $\pm$ sd} & \textbf{stat} & \textbf{mean $\pm$ sd} & \textbf{stat} & \textbf{mean $\pm$ sd} & \textbf{stat} & \textbf{mean $\pm$ sd} & \textbf{stat}\\
\end{tabular}
\end{tiny}
\end{table}
\end{landscape}

Table~\ref{tab:indicators_n100} and Table~\ref{tab:indicators_n250} show detailed results with respect to the HV- and $\varepsilon$-indicators for all problem instances with $100$ and $250$ nodes, respectively. The tables show mean values (rounded to four digits), standard deviations (rounded to two digits) and results of pairwise Wilcoxon-Mann-Whitney tests (with Bonferroni-Holm $p$-value adjustment) at significance level of $\alpha=0.01$ for the subset of baseline algorithms UNIFORM and 1BEX and the sub-graph based SGS and  USGS operators. Peeking at the results, we can see that USGS and SGS outperform the baseline algorithms in $100\%$ of the cases and the results are significant. Moreover, the standard deviation is equal to $0.00$ indicating a high robustness.

Fig.~\ref{fig:convergence} shows exemplary HV-indicator convergence plots ($x$-axis on log-scale) for one instance of each class. Here, all algorithms were started with the same initial population and the reference set was formed as the non-dominated points from the union of the supported efficient solutions calculated by WEIGHTED SUM and all obtained Pareto-front approximation for the instance.
We show the best baseline algorithm 1BEX for comparison only to keep the number of different trajectories visually manageable. We can see the drastic improvement in performance for all sub-graph based algorithms. For the instances of classes C1, C2 and C4 the algorithms are very close to the reference set after only $n$ iterations. For class C3, the algorithms progress slower. The reason is likely the enormous number of Pareto-optimal solutions (cf. Section~\ref{sec:instances} and also the random walks in Fig.~\ref{fig:random_walk_gallery}). 

\begin{figure}[t]
    \centering
    \includegraphics[width=\columnwidth]{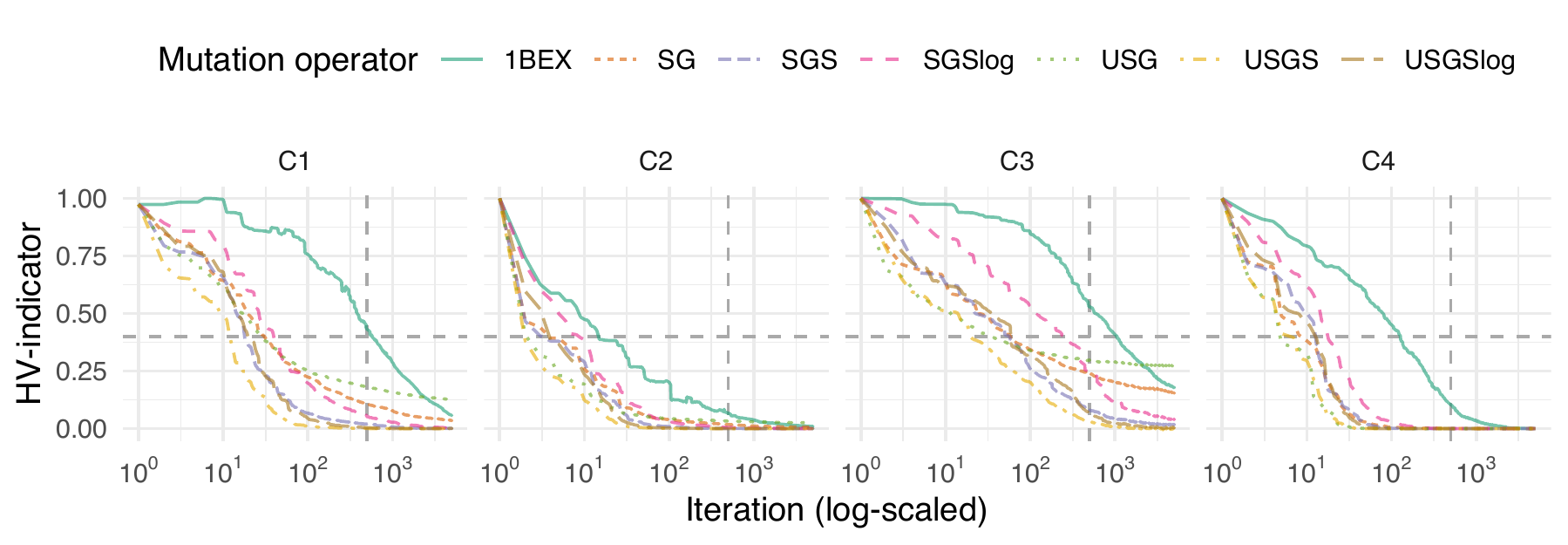}
    \caption{Convergence plots for different mutation operators on each one instance with $n = 100$ nodes from each class. The dashed vertical line indicates $10\%$ budget consumed. Note that the abscissa is on log-scale.}
    \label{fig:convergence}
\end{figure}

\begin{figure}[ht]
    \centering
    \includegraphics[width=\columnwidth]{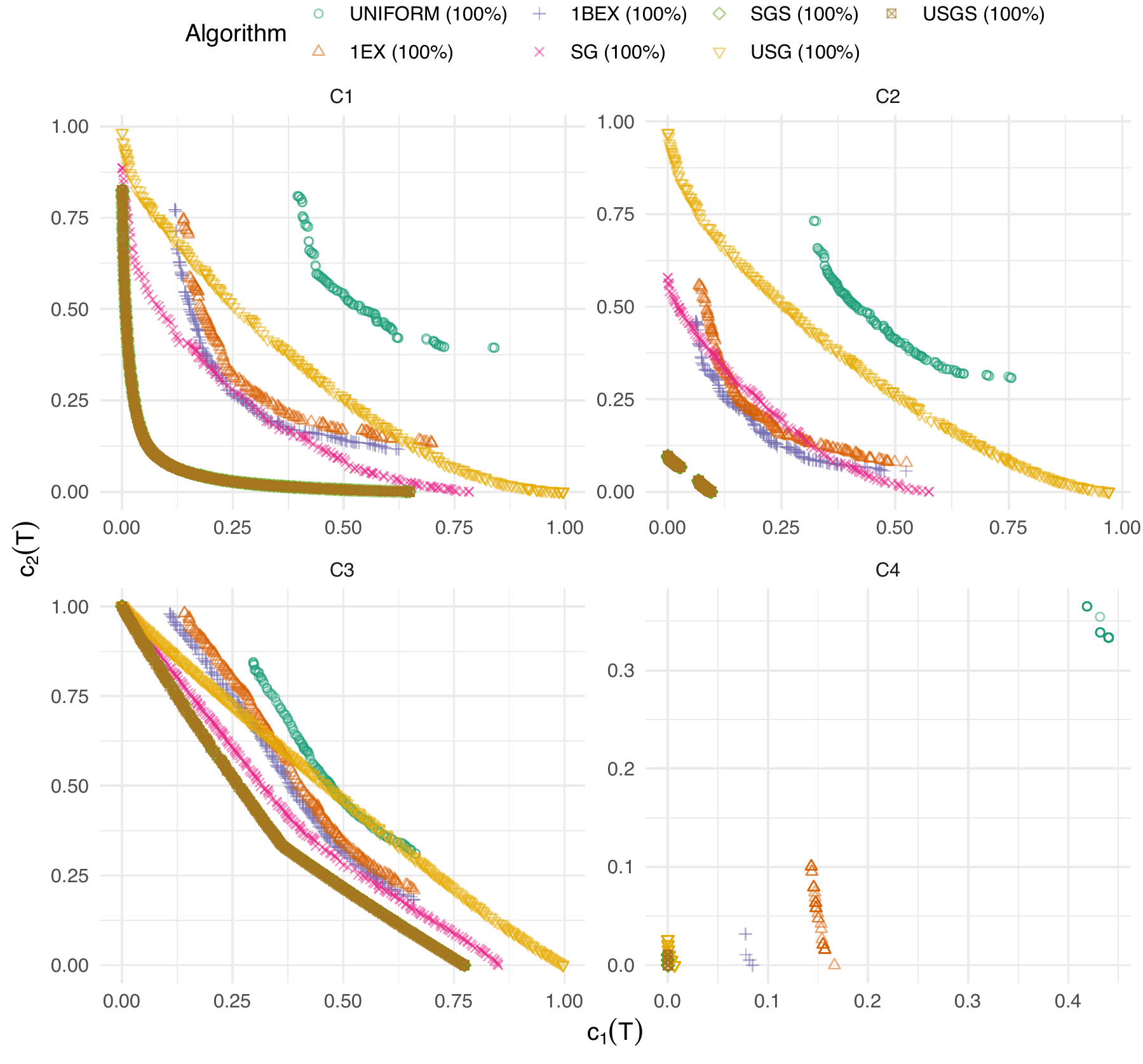}
    \caption{Approximations of the Pareto-front for each one randomly selected instance with $n = 250$ nodes from each considered graph class. We show the non-dominated points over each 30 independent runs of each algorithm. Note that solutions of SGS and USGS heavily overlap.}
    \label{fig:pareto_front_approximations}
\end{figure}

Fig.~\ref{fig:pareto_front_approximations} gives examples of approximated Pareto-fronts. Here, we can clearly see the large differences in the shape and quality of the produced approximations. In particular, it becomes even more obvious, that the scalarized versions (U)SGS are far superior to their counterparts (U)SG and the introduced modification to sample $\lambda$ from the interval $[0,1]$ is of major benefit.

\begin{figure}[t]
    \centering
    \includegraphics[width=\columnwidth]{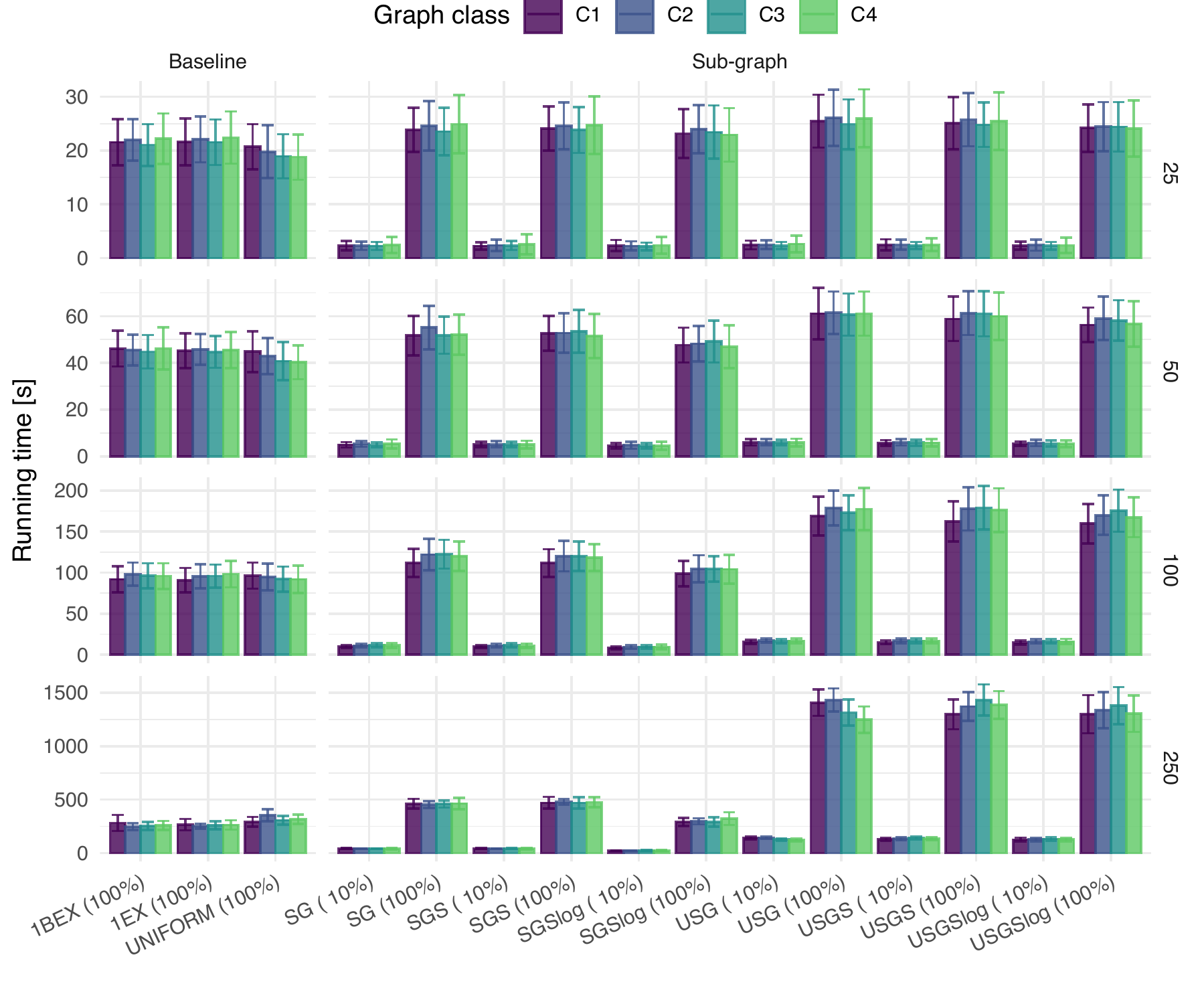}
    \caption{Average running times of all algorithms (with NSGA-II as encapsulating EMOA) in seconds split by instance size $n$ (rows) and colored by graph class.}
    \label{fig:running_times}
\end{figure}

Finally, Fig.~\ref{fig:running_times} empirically details the running times of the algorithms with integrated operators. The left column shows the baseline algorithms (1BEX, 1EX, and UNIFORM), while the average running times of all developed operators are presented in the right column for all instances and problem classes. All standard versions of the considered operators run longer than the baseline as expected.
Moreover, we can observe that with increasing $n$ -- in particular for $n=250$ -- SG(S) operates much faster than USG(S). The logarithmic versions ($\ast$log(100\%)), while approximately matching the running times of the baseline approaches, show different patterns. For USGS, the log-version does not reveal any advantage over the non-log version. For SGS, however, there is a slight advantage for increasing $n$. The advantage is quite low for our setup though, since the considered graphs are complete. Hence, $\Delta = n-1$ and the running time of SG(S) is upper bounded by $O(\sigma \Delta) = O((\log n)^2 n)$. Note that the running times reported are the average running times of the whole NSGA-II algorithm runs, not only the mutation operators. Hence, there is a non-negligible overhead for bookkeeping, survival selection (non-dominated sorting) etc.

It is clear and confirmed in Fig.~\ref{fig:running_times} that the 10\% variants of these approaches run significantly faster than baseline on the considered instance sizes. Thus, we can conclude that the sub-graph based approaches are (1) very competitive in finding high quality solutions, while at the same time (2) matching or even outperforming the running time of the simple baseline operators. Clearly, due to the more complex procedure of computing the sub-graph mutation, there will be a break-even in running time for the new mutation operators and the best baseline (1BEX in $O(n)$) for dense graphs with $n > 250$ and $\sigma = \Theta(n)$.
%For dense graphs with $n > 250$ and $\sigma = O(n)$ the algorithms will certainly take longer than the best baseline (1BEX in $O(n)$) due to their runtime complexity. 
However, on sparse graphs the computational overhead will be negligible in particular given the rapid convergence behavior.

%%%%%%%%%%%%%%%%%%%%%%
%%%%% CONCLUSION %%%%%
%%%%%%%%%%%%%%%%%%%%%%

\section{Conclusion}
\label{sec:conclusion}
%\input{chapters/conclusion}

%\jakob{Polish conclusion and focus on contributions mentioned at the beginning;  drop runtime stuff. Do not repeat stuff.}
In this paper we introduced heavily biased mutation operators, which are designed to solve the multi-objective minimum spanning tree (moMST) problem. These operators can be applied inside any given evolutionary multi-objective algorithm (EMOA) that relies on mutation. The core design idea for the operators stems from an analysis of the composition and neighborhood structure of (supported) efficient spanning trees. In a nutshell all proposed versions of the so-called sub-graph based operators SG(S) and USG(S) (1) drop a random subset of the edges of a spanning tree and (2) reestablish the spanning tree property by applying a single-objective MST algorithm to an induced sub-graph of the input graph. Consequently, the operators are best described as hybrids of random mutation and local search with a focus on avoiding the removal of \enquote{good} edges, i.e., edges which are part of many/most Pareto-optimal spanning trees.
An extensive study performed on $160$ benchmark instances with different Pareto-front characteristics shows the superiority of our proposal when compared to classical evolutionary algorithms from the literature. The sub-graph based algorithms outperform the baseline EMOAs in terms of three multi-objective performance indicators (HV, $\varepsilon$-indicator and $\Delta p$) even for severely reduced computational budget in terms of function evaluations.

We see many avenues for future investigations. First of all, an extension to the many-objective MST problem is interesting. Here, the main challenge is to adapt the sampling of weights for the weighed sum scalarization in a meaningful way to guarantee uniformity. In addition, a modified version of the operators seems promising, e.g., for the multi-objective versions of the degree-constraint MST, the maximum leaf MST or the Steiner-Tree-Problem.
Among the designed operators, first and foremost, the computationally cheaper SGS operator seems to be an adequate candidate to build a Pareto-Local-Search-based optimization method as the random walk experiments in this paper suggest. Eventually, further algorithmic improvements, e.g., dynamic choices of the single-objective solver dependent on the density/sparsity of the induced sub-graph or dynamic adaptation of the number of dropped edges, are promising future research directions worth to pursue.

\section*{Acknowledgments}

J. Bossek and C. Grimme acknowledge support by the European Research Center for Information Systems (ERCIS).

\newpage{}
\appendix

\section{Implementation Details}
\label{sec:implementation_details}

The pseudo-codes in Section~\ref{sec:mutation} are designed primarily with readability in mind; they aim to present the core ideas of the operators on a high level; essential details required for an efficient implementation to reach the bounds in the time complexity proofs were omitted though. We provide much more detailed pseudo-code for the proposed sub-graph mutation operators in this section. 

\subsection{Sub-Graph Mutation}

\begin{algorithm}[H]
\caption{Sub-Graph Mutation (SG).}
\label{alg:sg_finegrained}
\begin{algorithmic}[1]   
\Require{Graph $G = (V, E, c = (c_1, c_2)^{\top})$, spanning tree $T$ of $G$, threshold $\sigma \in \{3, \ldots, n\}$, round $\in \{$true, false$\}$}
    \State $v$ $\gets$ Random node from $V$
    \State $s$ $\gets$ Random integer from $\{3, \ldots, \sigma\}$
    \State $V_s$ $\gets$ \Call{BFS}{$T, v, s$} \Comment{See Algorithm~\ref{alg:sg_finegrained_bfs}}
    \State $\lambda$ $\gets$ Random weight between $0$ and $1$
    \If{round}
        \State $\lambda$ $\gets$ \Call{round}{$\lambda$} \Comment{Special case as introduced in~\cite{BG2017ParetoBeneficial}}
    \EndIf
    \State $G'$ $\gets$ \Call{GetInducedGraph}{$V_s, G, \lambda$} \Comment{See Algorithm~\ref{alg:sg_finegrained_induced}}
    \State $E_s^{*}$ $\gets$ \Call{Kruskal}{$G'$} \Comment{Apply classic Kruskal and extract the edges}
    \State $T^{*}$ $\gets$ \Call{ReplaceEdges}{$T, V_s, E_s^{*}$} \Comment{See Algorithm~\ref{alg:sg_finegrained_mutant}}
    \State \Return{$T^{*}$}
\end{algorithmic}
\end{algorithm}

\begin{algorithm}[H]
\caption{BFS}
\label{alg:sg_finegrained_bfs}
\begin{algorithmic}[1] 
\Require{Spanning tree $T$ of $G$, start node $v \in V$, size limit $s \in \{3, \ldots, n\}$}
\State Initialise empty queue $Q$
\State Q.\Call{Enqueue}{$v$}
\State $visited[v]$ $\gets$ $1$
\State $V_s \gets \{v\}$ \Comment{Set of visited nodes realised, e.g., as a linked list}
\While{$Q$ is not empty}
    \State $v$ $\gets$ Q.\Call{Dequeue}{}
    \For{$w$ in $adj_T(v)$} \Comment{Traverse nodes adjacent to $v$ in $T$}
        \If{$visited[w] = 1$}
            \State \textbf{Next}
        \EndIf
            \State $visited[w]$ $\gets$ $1$
            \State $V_s$ $\gets$ $V_s \cup \{w\}$
            \If{$|V_s|$ equals $s$}
                \State \Return{$V_s$} \Comment{Quit once $s$ nodes have been visited}
            \EndIf
            \State Q.\Call{Enqueue}{$w$}
        %\EndIf
    \EndFor
\EndWhile
\end{algorithmic}
\end{algorithm}

We start with the SG mutation where Algorithm~\ref{alg:sg} essentially hides details of the sub-routines. We thus provide more detailed pseudo-code of (i) SG itself in Algorithm~\ref{alg:sg_finegrained}, (ii) the BFS-call in Algorithm~\ref{alg:sg_finegrained_bfs}, (iii) the generation of the induced subgraph $G'$ in Algorithm~\ref{alg:sg_finegrained_induced} and (iv) the the final generation of the mutant in Algorithm~\ref{alg:sg_finegrained_mutant}.

We now explain the details starting with the BFS-traversal in Algorithm~\ref{alg:sg_finegrained_bfs}. An important note is that BFS and the other algorithms have access to a global array \verb|visited| of length $n$ which is initialised once prior to the first mutation to \verb|visited|$[v] = 0$ for all $v \in V$. This array is used internally by the BFS algorithm to mark already visited nodes and later on it is used to query in constant time whether nodes were selected or not. In the pseudo-codes we interpret \verb|visited|$[v]=1$ as Boolean \texttt{TRUE} and \texttt{FALSE} otherwise. BFS initialises an empty first-in-first-out queue $Q$ and adds the starting node $v$ to the queue and to a set (implemented, e.g., via a linked-list) $V_s$. Alongside, $v$ is marked as visited. Then, as long as the queue is not empty, a node $v$ is dequeued and $v$'s adjacency list in $G$ is traversed. If neighbor $w$ is spotted for the first time, it is added to $V_s$ and marked as visited. Once $s$ nodes have been visited the algorithm terminates prematurely and returns $V_s$ (and implicitly \verb|visited|).

\begin{algorithm}[H]
\caption{\textsc{GetInducedSubgraph}}
\label{alg:sg_finegrained_induced}
\begin{algorithmic}[1] 
\Require{Node set $V_s$, graph $G = (V, E, c = (c_1, c_2)^{\top})$, scalar weight $\lambda \in [0,1]$}
\State $G'$ $\gets$ $(V_s, E' = \emptyset)$
\For{$v$ in $V_s$}
    \For{$w$ in $adj_G(v)$} \Comment{Iterate over all neighbors of $v$ in $G$}
        \State $e \gets \{v, w\} \in E$
        \If{$visited[v]$ \textbf{ and } $visited[w]$}
            \State $E'$ $\gets$ $E' \cup \{e\}$
            \State $c^s(e)$ $\gets$ $\lambda c_1(e) + (1-\lambda) c_2(e)$ \Comment{Scalarise edge weights}
        \EndIf
    \EndFor
\EndFor
\State \Return{$G'$}
\end{algorithmic}
\end{algorithm}

Once BFS is done the SG-mutation in Algorithm~\ref{alg:sg_finegrained} samples the weight for scalarisation and continues building the induced sub-graph $G'$. The details of the latter process are given in Algorithm~\ref{alg:sg_finegrained_induced}: $G'$ is initialised with $V_s$ as its node set and an empty edge set. Next, $V_s$, the set of sampled nodes, is traversed and for each $v \in V_s$ the inner-loop traverses $v$'s neighbors in the input graph $G$. If both nodes are marked as visited the edge is added to $G'$ and the edge weight is scalarised. Note that the constant-time checks via the \verb|visited| array are essential for efficiency. Note further that a straight-forward mapping is required to map the nodes from $G'$ to their counterparts in $G$ and vice versa.

\begin{algorithm}[H]
\caption{\textsc{ReplaceEdges}}
\label{alg:sg_finegrained_mutant}
\begin{algorithmic}[1] 
\Require{Parent ST $T$, subset of nodes $V_s$, optimal edges set $E_s^{*}$ between nodes in $V_s$}
\State $T' \gets (V, \emptyset{})$ \Comment{Initialise empty spanning tree}
\For{$e = \{v, w\} \in E(T)$} \Comment{Copy untouched edges of the parent tree}
    \If{\textbf{not } visited$[v]$ \textbf{ or } \textbf{not }visited$[w]$}
        \State $E(T') \gets E(T') \cup \{e\}$
    \EndIf
\EndFor
\State $E(T') \gets E(T') \cup E_s^{*}$ \Comment{Add spanning tree edges of the induced sub-graph to complete the solution}
\State \Return{$T'$}
\end{algorithmic}
\end{algorithm}

Next, Kruskal's algorithm is applied to $G'$ without modifications. Eventually, the mutant is put together by first copying all edges which are not part of the induced sub-graph $G'$ from the parent to the mutant and subsequently adding the spanning-tree edges of the induced sub-graph (see Algorithm~\ref{alg:sg_finegrained_mutant}). This step completes the mutation.

\subsection{Unconnected Sub-Graph Mutation}

\begin{algorithm}[t]
\caption{Unconnected Sub-graph Mutation (USG)}
\label{alg:usg_finegrained}
\begin{algorithmic}[1]
    \Require{Graph $G = (V, E, c = (c_1, c_2)^{\top})$, spanning tree $T$ of $G$,\newline{} threshold $\sigma \in \{1, \ldots, n-1\}$, round $\in \{$true, false$\}$}
    \State $s$ $\gets$ Random integer from $\{1, \ldots, \sigma\}$
    \State $D$ $\gets$ array of length $n-1$ initialised to zero entries \Comment{$D[i] = 1$ indicates that the $i$th edge should be \underline{dropped} from $T$}
    \State $D[i] \gets 1$ for $s$ randomly sampled positions $i \in \{i_1, \ldots, i_s\}$
    \State $T' \gets (V, E')$ with $e_i = \{u,v\} \in E(T')$ if $D[i] = 0$ \Comment{$T'$ is a spanning forest, i.e., a set of trees without the dropped edges; $T'$ will later be augmented with edges to build of spanning tree of $G$}
    \State Initialise union find data structure $UF$ on $V = \{1, \ldots, n\}$
    \State $UF$.\Call{Union}{$v$, $w$} for $\{v,w\} \in T'$ \Comment{I.e., union all components already linked in the spanning forest $T'$}
    \State $\lambda$ $\gets$ Random weight between $0$ and $1$
    \If{round}
        \State $\lambda$ $\gets$ \Call{round}{$\lambda$}
    \EndIf
    \State $G'$ $\gets$ $(V, E, \lambda c_1 + (1 - \lambda) c_2)$ \Comment{Scalarise all edge weights in $G$}
    \State $T^{*}$ $\gets$ \Call{Kruskal}{$G', T', UF$} \Comment{Apply Kruskal to make a ST based on $T'$}
    \State \Return{$T^{*}$}
\end{algorithmic}
\end{algorithm}

This algorithm is simpler than SG implementation-wise. Nevertheless, we feel that the additional details in Algorithm~\ref{alg:usg_finegrained} are beneficial for understanding. The algorithm first samples a random integer $s \in \{1, \dots, \sigma\}$ which defines the number of edges that will be dropped from the parent spanning tree $T = (V, E_T)$. Next, a utility data-structure $D$ is initialised as an array of length $n-1$ with all zero entries and $s$ random entries are set to $1$. The edges $e_i$ corresponding to these entries are now dropped from $T$ in the next step to produce $T' = (V, E')$. Since $T$ is a spanning tree of $G$, removing $s$ edges from $T$ results in $T'$ being a \emph{spanning forest}, i.e., a set of trees, with $s+1$ connected components. $T'$ serves as a starting point for a modified version of Kruskal's MST algorithm which needs to connect these $s+1$ components into a single one. To this end, the next step is to initialise a Union-Find data-structure that supports efficient union and find operations on sets~(see, e.g., \cite{cormen2009IntroductionToAlgorithms}). We first initialise the union find data-structure $UF$ on the set of nodes $V=\{1, \ldots, n\}$ of the input graph $G$. Subsequently, we iterate over all non-excluded edges $e_i = \{v,w\} \in T'$, i.e., edges $e_i$ with $D[i]=0$, and union the corresponding sets prior to the Kruskal-run; these ensures that non-excluded edges remain in the mutant. The final preparation step is to make a copy $G'$ of the input graph $G$ where all edges have scalar weights. Here, the sampled weight $\lambda \in [0,1]$ comes into play. Eventually, Kruskal's algorithm is started with the initial spanning forest $T'$ and the Union-Find data-structure $UF$ that encodes the interconnection structure of $T'$ as parameters alongside the edge-scalarised graph $G'$. The resulting spanning tree $T^{*}$ is the mutant.

\section{More details on the experimental setup}
\label{sec:experiments_details}

Our experimental pipeline is implemented in the statistical programming language R~\citep{Rlanguage}. The package \pkg{grapherator}~\citep{B2018grapherator} was used to generate the benchmark instances sets. Package \pkg{mcMST}~\citep{Bo17} was extended by the proposed mutation operators; the package itself is written in R, but the time-critical operators are implemented in C++ for performance reasons. 
For multi-objective performance assessment we used the evolutionary computation framework \pkg{ecr}~\citep{B2017ecr}. Due to the multitude of instances, we run experiments in parallel on the high performance cluster PALMA2 of the University of M\"unster. The calculations were performed on Skylake 
(Gold 6140) compute nodes with 4GB of RAM. To ease the job-management we used the \pkg{batchtools}~\citep{LBS2017_batchtools} library.

\subsection{Data repository}

Code, data, and additional plots can be found in a public Github repository: \url{https://github.com/jakobbossek/ECJ-MOMST-Subgraph-Mutation}. This repository contains a couple of folders whose content is described briefly in the following. More information can be found in the repository.
\begin{description}
    \item[\texttt{figures/}] Additional plots
    \begin{description}
        \item[\texttt{figures/apriori\_analysis/}] contains heatmaps and graph embeddings supporting the observations made in Section~\ref{sec:analysis}.
        \item[\texttt{figures/random\_walks/}] contains further examples of random walks with different parameter settings (see Section~\ref{sec:mutation}).
        \item[\texttt{figures/benchmark/}] all remaining plots (instances, heat-maps, etc.) supporting Section~\ref{sec:experiments}.
    \end{description}
    \item[\texttt{instances/}] the set of benchmark instances.
    \item[\texttt{src/}] R source files shared by multiple experiments.
\end{description}
Moreover, an implementation of the algorithms in C++ can be found in the R package \pkg{mcMST}~\citep{Bo17}.

\section{Supplementary experimental results}

Fig.~\ref{fig:pairwise_tests_eps_deltap} visualises the fraction of non-parametric Wilcoxon-Mann-Whitney tests for each pair of algorithms/operators for the $\varepsilon$-indicator and the $\Delta p$-indicator. The observations are very similar to what we saw in Fig.~\ref{fig:pairwise_tests_HV}.
Tables~\ref{tab:indicators_n25} and \ref{tab:indicators_n50} show fine-grained results on the HV-indicator and the $\varepsilon$-indicator on all instances of size 25 and 50 respectively (analogous to Tables~\ref{tab:indicators_n100} and \ref{tab:indicators_n250}).

\begin{landscape}
\begin{figure}[h!]
    \centering
    \includegraphics[width=0.77\textwidth]{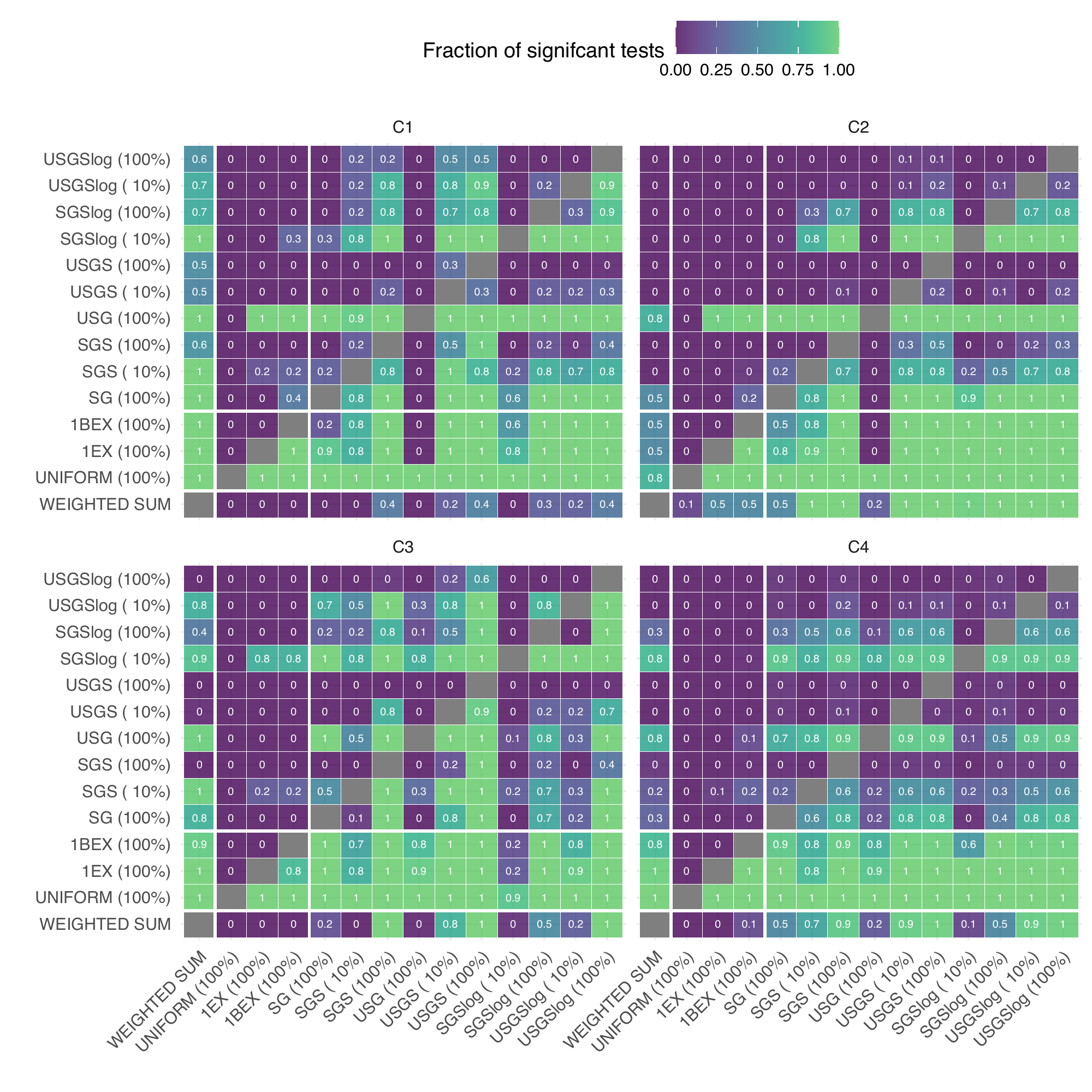}
    \includegraphics[width=0.77\textwidth]{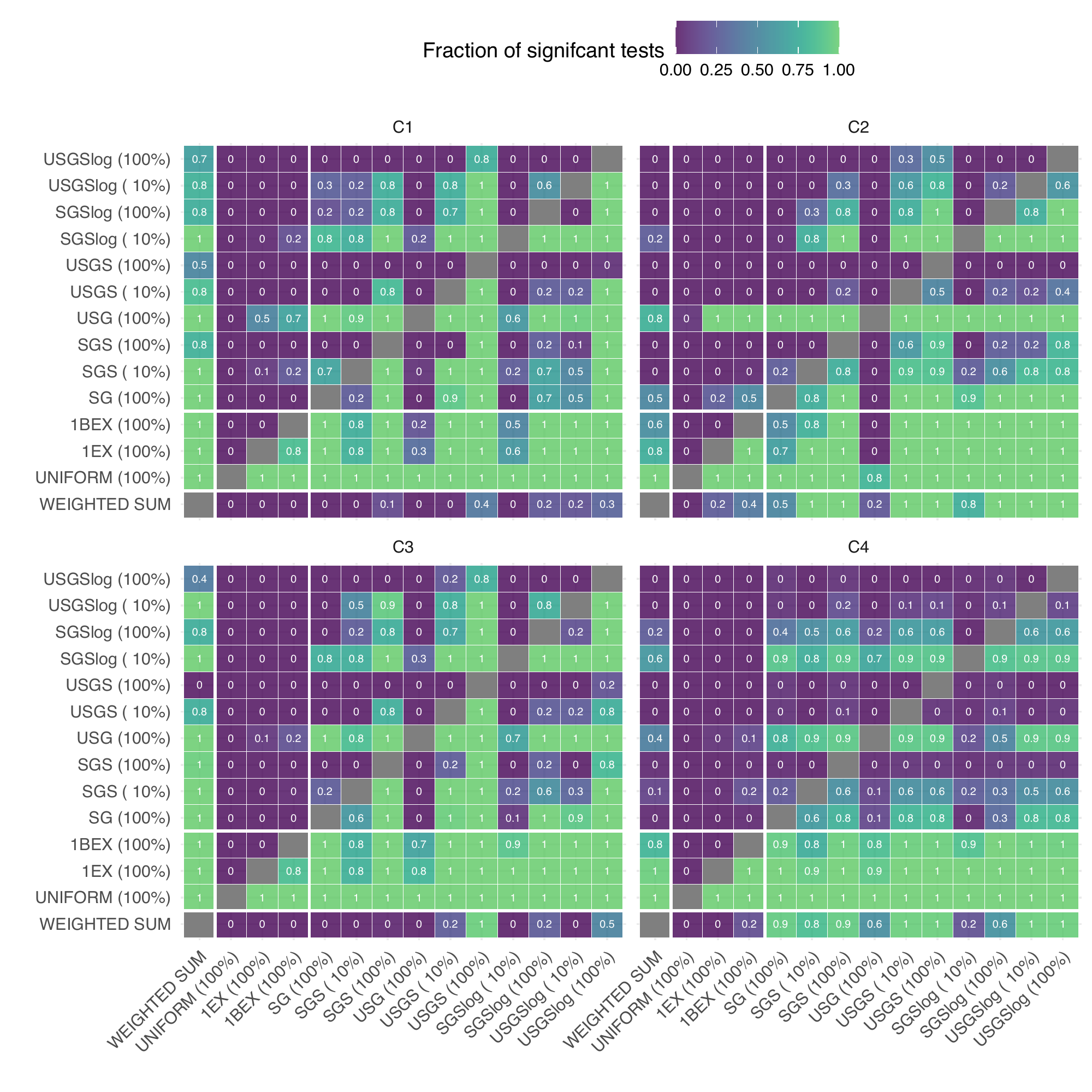}
    \caption{Visualisation of the fraction of pairwise Wilcoxon-Mann-Whitney tests to check whether the column-algorithm has a significantly $(\alpha = 0.01$) lower $\epsilon$-indicator (left) or $\Delta p$-indicator (right) median value than the row-algorithm. Results are split by graph class. Bonferroni-Holm $p$-value adjustments was applied in order to avoid multiple-testing issues.}
    \label{fig:pairwise_tests_eps_deltap}
\end{figure}
\end{landscape}

% \begin{figure}[h!]
%     \centering
%     \includegraphics[width=\columnwidth]{figures/pwtests_EPS.pdf}
%     \caption{Visualisation of the fraction of pairwise Wilcoxon-Mann-Whitney tests to check whether the column-algorithm has a significantly $(\alpha = 0.01$) lower $\epsilon$-indicator median value than the row-algorithm. Results are split by graph class. Bonferroni-Holm $p$-value adjustments was applied in order to avoid multiple-testing issues.}
%     \label{fig:pairwise_tests_eps}
% \end{figure}
% \begin{figure}[h!]
%     \centering
%     \includegraphics[width=\columnwidth]{figures/pwtests_DeltaP.pdf}
%     \caption{Visualisation of the fraction of pairwise Wilcoxon-Mann-Whitney tests to check whether the column-algorithm has a significantly $(\alpha = 0.01$) lower $\Delta p$-indicator median value than the row-algorithm. Results are split by graph class. Bonferroni-Holm $p$-value adjustments was applied in order to avoid multiple-testing issues.}
%     \label{fig:pairwise_tests_DeltaP}
% \end{figure}

\begin{landscape}
\begin{table}

\caption{\label{tab:indicators_n25}Results for instances with 25 nodes: Mean, standard deviation~(\textbf{sd}) and results of Wilcoxon-Mann-Whitney tests at significance level $\alpha=0.01$ (\textbf{stat}) with respect to HV-indicator and $\varepsilon$-indicator respectively. The \textbf{stat}-column is to be read as follows: a value $X^{+}$ indicates that the indicator for the column algorithm (note that algorithms are numbered and color-encoded in the second row) is significantly lower than the one of algorithm $X$. Lowest indicator values are highlighted in \textbf{bold-face}.}
\centering
\begin{tiny}
\renewcommand{\arraystretch}{1}
\renewcommand{\tabcolsep}{2.8pt}
\begin{tabular}[t]{rrrrrrrrrrrrrrrrrrr}
\toprule
\multicolumn{1}{c}{\textbf{ }} & \multicolumn{1}{c}{\textbf{ }} & \multicolumn{1}{c}{\textbf{ }} & \multicolumn{8}{c}{\textbf{HV-indicator}} & \multicolumn{8}{c}{\textbf{$\varepsilon$-indicator}} \\
\cmidrule(l{3pt}r{3pt}){4-11} \cmidrule(l{3pt}r{3pt}){12-19}
\multicolumn{1}{c}{\textbf{ }} & \multicolumn{1}{c}{\textbf{ }} & \multicolumn{1}{c}{\textbf{ }} & \multicolumn{2}{c}{\textbf{UNIF (\textcolor{violet}{1})}} & \multicolumn{2}{c}{\textbf{1BEX (\textcolor{brown}{2})}} & \multicolumn{2}{c}{\textbf{SGS (\textcolor{teal}{3})}} & \multicolumn{2}{c}{\textbf{USGS (\textcolor{purple}{4})}} & \multicolumn{2}{c}{\textbf{UNIF (\textcolor{violet}{1})}} & \multicolumn{2}{c}{\textbf{1BEX (\textcolor{brown}{2})}} & \multicolumn{2}{c}{\textbf{SGS (\textcolor{teal}{3})}} & \multicolumn{2}{c}{\textbf{USGS (\textcolor{purple}{4})}} \\
\cmidrule(l{3pt}r{3pt}){4-5} \cmidrule(l{3pt}r{3pt}){6-7} \cmidrule(l{3pt}r{3pt}){8-9} \cmidrule(l{3pt}r{3pt}){10-11} \cmidrule(l{3pt}r{3pt}){12-13} \cmidrule(l{3pt}r{3pt}){14-15} \cmidrule(l{3pt}r{3pt}){16-17} \cmidrule(l{3pt}r{3pt}){18-19}
 & $n$ & $I$ & \textbf{mean $\pm$ sd} & \textbf{stat} & \textbf{mean $\pm$ sd} & \textbf{stat} & \textbf{mean $\pm$ sd} & \textbf{stat} & \textbf{mean $\pm$ sd} & \textbf{stat} & \textbf{mean $\pm$ sd} & \textbf{stat} & \textbf{mean $\pm$ sd} & \textbf{stat} & \textbf{mean $\pm$ sd} & \textbf{stat} & \textbf{mean $\pm$ sd} & \textbf{stat}\\
\bottomrule
\end{tabular}
\end{tiny}
\end{table}

\end{landscape}

\begin{landscape}
\begin{table}

\caption{\label{tab:indicators_n50}Results for instances with 50 nodes: Mean, standard deviation~(\textbf{sd}) and results of Wilcoxon-Mann-Whitney tests at significance level $\alpha=0.01$ (\textbf{stat}) with respect to HV-indicator and $\varepsilon$-indicator respectively. The \textbf{stat}-column is to be read as follows: a value $X^{+}$ indicates that the indicator for the column algorithm (note that algorithms are numbered and color-encoded in the second row) is significantly lower than the one of algorithm $X$. Lowest indicator values are highlighted in \textbf{bold-face}.}
\centering
\begin{tiny}
\renewcommand{\arraystretch}{1}
\renewcommand{\tabcolsep}{3.7pt}
\begin{tabular}[t]{rrrrrrrrrrrrrrrrrrr}
\toprule
\multicolumn{1}{c}{\textbf{ }} & \multicolumn{1}{c}{\textbf{ }} & \multicolumn{1}{c}{\textbf{ }} & \multicolumn{8}{c}{\textbf{HV-indicator}} & \multicolumn{8}{c}{\textbf{$\varepsilon$-indicator}} \\
\cmidrule(l{3pt}r{3pt}){4-11} \cmidrule(l{3pt}r{3pt}){12-19}
\multicolumn{1}{c}{\textbf{ }} & \multicolumn{1}{c}{\textbf{ }} & \multicolumn{1}{c}{\textbf{ }} & \multicolumn{2}{c}{\textbf{UNIF (\textcolor{violet}{1})}} & \multicolumn{2}{c}{\textbf{1BEX (\textcolor{brown}{2})}} & \multicolumn{2}{c}{\textbf{SGS (\textcolor{teal}{3})}} & \multicolumn{2}{c}{\textbf{USGS (\textcolor{purple}{4})}} & \multicolumn{2}{c}{\textbf{UNIF (\textcolor{violet}{1})}} & \multicolumn{2}{c}{\textbf{1BEX (\textcolor{brown}{2})}} & \multicolumn{2}{c}{\textbf{SGS (\textcolor{teal}{3})}} & \multicolumn{2}{c}{\textbf{USGS (\textcolor{purple}{4})}} \\
\cmidrule(l{3pt}r{3pt}){4-5} \cmidrule(l{3pt}r{3pt}){6-7} \cmidrule(l{3pt}r{3pt}){8-9} \cmidrule(l{3pt}r{3pt}){10-11} \cmidrule(l{3pt}r{3pt}){12-13} \cmidrule(l{3pt}r{3pt}){14-15} \cmidrule(l{3pt}r{3pt}){16-17} \cmidrule(l{3pt}r{3pt}){18-19}
 & $n$ & $I$ & \textbf{mean $\pm$ sd} & \textbf{stat} & \textbf{mean $\pm$ sd} & \textbf{stat} & \textbf{mean $\pm$ sd} & \textbf{stat} & \textbf{mean $\pm$ sd} & \textbf{stat} & \textbf{mean $\pm$ sd} & \textbf{stat} & \textbf{mean $\pm$ sd} & \textbf{stat} & \textbf{mean $\pm$ sd} & \textbf{stat} & \textbf{mean $\pm$ sd} & \textbf{stat}\\
\bottomrule
\end{tabular}
\end{tiny}
\end{table}
\end{landscape}

\small
\bibliographystyle{apalike}
\bibliography{bib}

\end{document}